\documentclass[11pt, letterpaper]{article}

\usepackage{deepthink}
\usepackage{lipsum} 

\newtheorem{theorem}{Theorem}
\newtheorem{corollary}{Corollary}
\usepackage{mathtools}
\usepackage{standalone}
\newenvironment{proof}{\par\noindent\textit{Proof.}\ }{\hfill$\square$\par}
\newtheorem{example}{Example}[section]

\usepackage{graphicx}
\usepackage{fancybox}
\usepackage{amsmath}
\usepackage{bm}
\usepackage{subcaption}
\usepackage{algorithm}
\usepackage{algorithmic}
\usepackage{xcolor}
\usepackage{verbatim}
\usepackage{url}
 \usepackage{hyperref}

\usepackage{enumitem}
\usepackage{cleveref}
\usepackage{wrapfig}
\usepackage{tikz}
\usepackage{booktabs}
\usetikzlibrary{calc} 

\usepackage[font={footnotesize,sf},labelfont=bf]{caption}

\usepackage{tikz-3dplot}
\usetikzlibrary{decorations.text}

\newcommand{\R}{\mathbb{R}}

\DeclareMathOperator{\sspan}{span}

\newcommand{\revision}[1]{\textcolor{black}{#1}}
\newcommand{\nextrev}[1]{\textcolor{black}{#1}}

\newcommand{\laura}[1]{\textcolor{magenta}{\bf [{\em LB:} #1]}}

\definecolor{tianjiaogold}{rgb}{0.84, 0.62, 0.00}
\newcommand{\tianjiaogold}[1]{\textcolor{tianjiaogold}{#1}}

\definecolor{tianjiaored}{rgb}{0.71, 0.32, 0.30}
\newcommand{\tianjiaored}[1]{\textcolor{tianjiaored}{#1}}

\definecolor{tianjiaoblue}{rgb}{0.42, 0.56, 0.75}

\def\W{\bm{W}}

\def\z{\bm{z}}
\def\E{\mathbb{E}}
\def\Y{\bm{Y}}
\def\X{\bm{X}}
\def\Diag{\text{Diag}}
\def\Z{\bm{Z}}

\newcounter{Lcount}

\usepackage[
  backend=biber,
  style=alphabetic,
  maxbibnames=100,
  minbibnames=100,
  maxcitenames=2,
  mincitenames=2
]{biblatex}
\title{An Overview of Low-Rank Structures in the Training and Adaptation of Large Models}

\authorblock{
    \href{https://web.eecs.umich.edu/~girasole/}{Laura Balzano}\textsuperscript{1}, \href{https://tianjiaoding.com/}{Tianjiao Ding}\textsuperscript{2}, \href{https://www.cis.jhu.edu/~haeffele/}{Benjamin D. Haeffele}\textsuperscript{2}, \href{https://soominkwon.github.io/}{Soo Min Kwon}\textsuperscript{1}, \href{https://qingqu.engin.umich.edu/}{Qing Qu}\textsuperscript{1}, \href{https://peng8wang.github.io/}{Peng Wang}\textsuperscript{3},\\ \href{https://www.vita-group.space/}{Zhangyang Wang}\textsuperscript{4}, \href{https://canyaras.com/}{Can Yaras}\textsuperscript{1}
}

\affiliation{
    \textsuperscript{1}University of Michigan \; $\cdot$ \; \textsuperscript{2}University of Pennsylvania \; $\cdot$ \; \textsuperscript{3}University of Macau  \; $\cdot$ \; \textsuperscript{4}UT Austin
}

\authornote{
    Authors are listed alphabetically
}

\abstracttext{

\noindent The substantial computational demands of modern large-scale deep learning present significant challenges for efficient training and deployment. Recent research has revealed a widespread phenomenon wherein deep networks inherently learn low-rank structures in their weights and representations during training. This tutorial paper provides a comprehensive review of advances in identifying and exploiting these low-rank structures, bridging mathematical foundations with practical applications. We present two complementary theoretical perspectives on the emergence of low-rankness: viewing it through the optimization dynamics of gradient descent throughout training, and understanding it as a result of implicit regularization effects at convergence. Practically, these theoretical perspectives provide a foundation for understanding the success of techniques such as Low-Rank Adaptation (LoRA) in fine-tuning, inspire new parameter-efficient low-rank training strategies, and explain the effectiveness of masked training approaches like dropout and masked self-supervised learning.
}


\date{\today}

\correspondence{\href{mailto:girasole@umich.edu}{girasole@umich.edu}, \href{mailto:tjding@upenn.edu}{tjding@upenn.edu}, \href{mailto:pengw@um.edu.mo}{pengw@um.edu.mo}}

\resources{\;\;
    \href{https://github.com/girasole/SPM-LowRankDeepLearning}{Code Repository}
}

\headerlogo{}{}

\begin{document}

\makeDeepthinkHeader 
\begin{figure}[h]
\vspace{-0.6in}
\includegraphics[width=\linewidth]{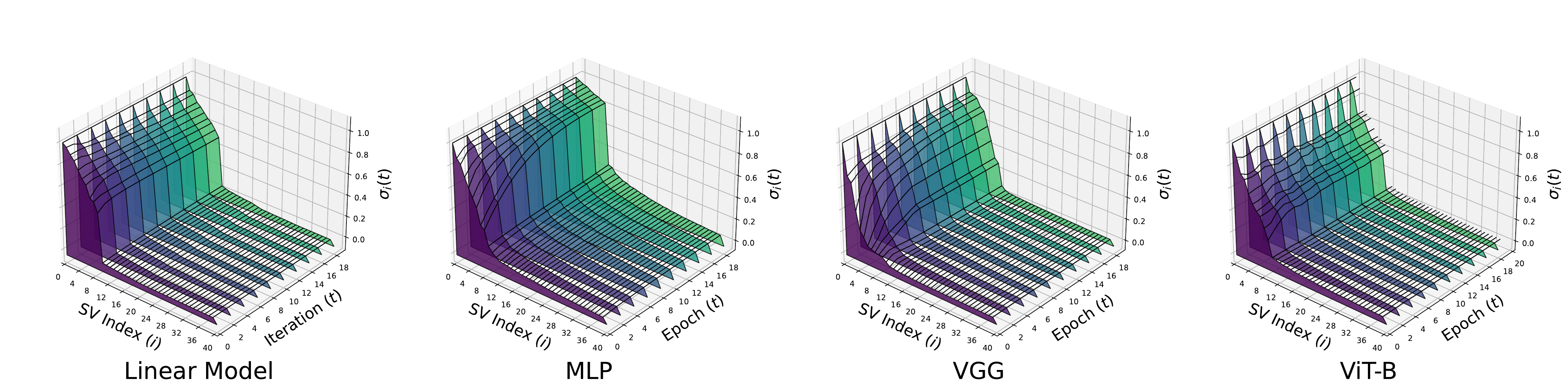}
     \caption{\textbf{Prevalence of low-rank weight updates in various deep networks.} Each plot visualizes the
     singular values of the weight updates from initialization for the penultimate layer weight matrix for different types of network architectures: deep linear network (DLN), multi-layer perception (MLP), VGG, and ViT-B. 
     We show them with two views: one in 3D emphasizes the evolution of the singular values over iteration, and the other using color shows the singular values on log-axis giving a clearer picture of the low-rank structure in each setting.
     The linear network is trained on MNIST with mean square error loss. The MLP is trained on MNIST with cross-entropy loss.
     The VGG and ViT-B networks are both trained on CIFAR-10 with cross-entropy loss. 
     The result shows a prevalent phenomenon across linear and nonlinear networks -- gradient descent only updates a small portion of the singular values, while the others remain small and almost unchanged. Courtesy of \cite{kwon2023efficient}. \nextrev{A 2D figure with log y-axis can be found in Supplementary \Cref{app:additional-figures}.} }
    \label{fig:svals}
\end{figure}

\newpage

\tableofcontents

\newpage

\section{Introduction}\label{sec:intro}

The advent of deep learning and large-scale computing has immeasurably changed the ways we process, interpret, and predict data in signal processing and machine learning. However, training and deploying modern deep learning models demand substantial computational resources, raising concerns about exorbitant training costs, GPU shortages, and heightened energy consumption in the coming years. Additionally, our theoretical understanding of how deep learning works is limited compared to that of classical signal processing methods. Decades of research in signal processing and statistical inference have characterized the sample complexity of accurate parameter estimation across over- and under-determined settings, leveraging a variety of data structures and priors 
\cite{procIEEE2010compressedsensing, wright2022high} (e.g., low-rank structure and sparsity).
However, for most state-of-the-art deep learning models, the number of parameters often far exceeds the available data samples, \revision{and explicit regularization or signal structure is not necessary to achieve impressive learning performance. This tutorial paper will survey recent work in implicit or emergent low-rank structure in deep learning, and how to leverage it for efficient algorithms while maintaining performance.}



Within the general study of the mathematics of deep learning, several lines of research over the last decade have explored the emergence of low-dimensional structures during the training process, where basic elements such as weight matrices, adapters, gradients, and activations (i.e., outputs of certain layers of a neural network, sometimes also called features) 
tend to be approximately low-rank even without explicit regularization. 
These low-dimensional structures arise in part due to the implicit bias of the methods used to train deep networks, 
providing the potential to partially explain why deep models require fewer samples than the number of model parameters. 
Although implicit bias varies with the training algorithm and model architecture \cite{belkin2019reconciling,huh2023the,soudry_implicit_2018} (e.g., margin maximization, minimum-norm interpolation, and sparsity and spectral bias), our analysis centers on low-rank bias, which enables prominent applications in efficient training and fine-tuning.

This implicit low-dimensionality has inspired the exploration of low-rank structures for more efficient training and fine-tuning of large-scale deep learning models. For example, low-rank adaptation (LoRA) \cite{hu2021lora} adds a low-rank factorized update to pre-trained weight matrices for model fine-tuning, using significantly reduced computation and memory. It has gained significant attention in the past few years due to its impressive performance in fine-tuning large language models (LLMs), vision-language models, and deep generative models for vision. 
Inspired by this success, deep learning practitioners further explicitly factorize weight matrices into low-rank factors for post-training compression, as well as train low-rank models from scratch, with promising results. 
For instance, recent advances made by the DeepSeek-V3 model \cite{liu2024deepseek} have achieved impressive language generation results with significantly compressed model sizes, in part by using a low-rank factorization of the query matrices in the multi-head attention during training and inference. 
In this paper, we review recent exciting advances and aim to clarify the mathematical foundations underlying their design. Specifically, we highlight key insights from a rich line of research focused on theoretically understanding and leveraging low-rank structures in deep learning. 

\begin{figure}[t]
    \centering
    \includegraphics[width=0.88\linewidth]{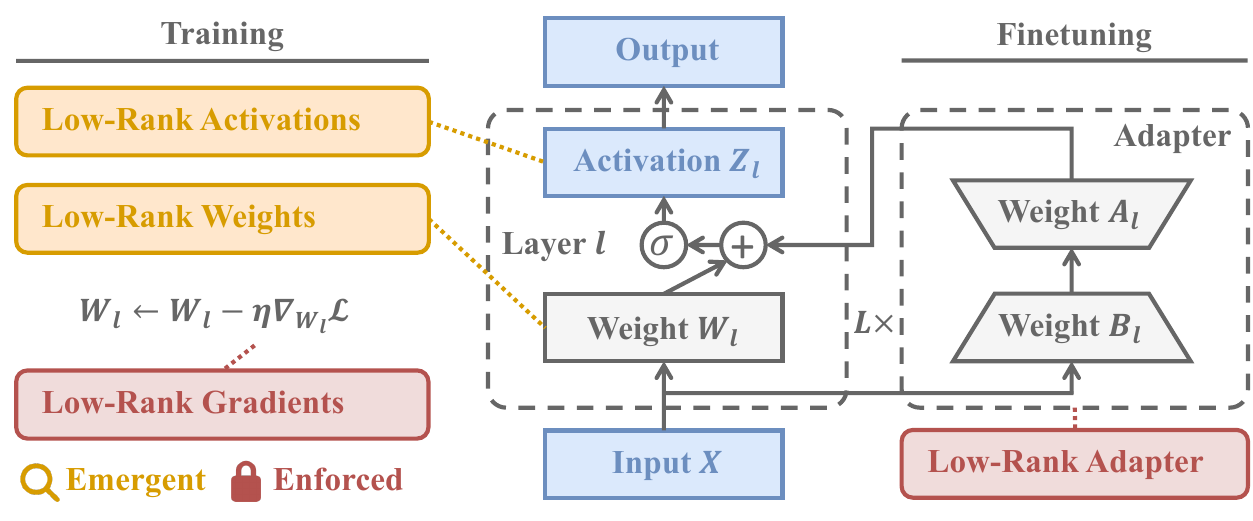}
    \caption{\textbf{Top}: We focus on low-rank structure found in the \textbf{Weights, Adapters, Gradients, and Activations} of deep networks. \textbf{Bottom}: Organization of the paper.}
    \label{fig:nomenclature_organization}

    \vspace{1em} 
    \begin{tabular}{lll}
    \hline
    \S \ref{sec:background} &  & Optimization Basics of Deep Networks and Low-Rank Structures \\
     & \S \ref{subsec:deep_networks}, \ref{subsec:low_rank_bg} & \quad Deep networks; low-rank data matrix and its factorization \\
     & \S \ref{subsec:low_rank_bg} & \quad Example of enforced low-rank structure in matrix factorization \\
     & \S \ref{sec:imp_struct} & \quad Example of emergent structure in gradient dynamics \\ \hline
    \S \ref{sec:LRtrain} &  & Low-Rank Structures at \textit{Every Iteration} of Gradient Methods \\
     & \S \ref{sec:train_dyn} & \quad \tianjiaogold{Emergent low-rank layer weights} during training \\
     & \S \ref{sec:lora} & \quad \tianjiaored{Enforced low-rank adapter weights} for efficient finetuning \\
     & \S \ref{sec:low-rank-training} & \quad \tianjiaored{Enforced low-rank gradients} for efficient training \\ \hline
    \S \ref{sec:LRobj} &  & Low-Rank Structures at \textit{Convergence} of Gradient Methods \\
     & \S \ref{sec:var_form}, \ref{sec:low_rank_imp_reg} & \quad \tianjiaogold{Emergent low-rank activations} from $l_2$ regularization \\
     & \S \ref{sec:mask_train} & \quad \tianjiaogold{Emergent low-rank activations} from masked training \\ \hline
    \end{tabular}
    \label{tab:org-table}
\end{figure}
\noindent 
\revision{{\bf Different notions of low-rank structure.}  In deep learning, the term “low-rank structure” refers to several distinct phenomena, each arising from different aspects of model design, training dynamics, or adaptation strategies. In this paper, we mainly focus on four places where low-rank structures arise in deep neural networks: 
\begin{itemize}[nosep]
    \item {\em Weights}: Weight matrices in neural networks are low-rank. 
    \item {\em Adapters}: Fine-tuning adaptation methods restrict parameter updates to low-rank matrices. 
    \item {\em Gradients}: The gradients with respect to weight matrices are low-rank during training. 
    \item {\em Activations}\footnote{A neural network’s activation in a certain layer is the output after the nonlinear activation function for a given input. In comparison, representation is the information encoded by patterns of activations—built from activations but describing higher-level structure/content across units or layers.}: The output of certain layers is low-rank.
\end{itemize}
Together, these highlight complementary perspectives on how low-rank structure manifests in practice, from parameterization to optimization and representation. In \Cref{fig:nomenclature_organization}, we show an overview figure identifying each of these four structures in a canonical deep network layer. Below that, we list the outline of the paper, where we cover a variety of topics on emergent and enforced low-rank structure in training deep neural networks. }


\smallskip 
\noindent \revision{
{\bf Theoretical study of Deep Linear and Deep Nonlinear Networks.} This tutorial manuscript is for a special issue on the mathematics of deep learning. The mathematical theory of how low-rank structure emerges in the training and optimization of deep networks is novel and intriguing, yet it is almost entirely limited to the setting of deep linear networks. On one hand, the empirical evidence for low-rank structure in deep nonlinear networks is prevalent; \Cref{fig:svals} shows the singular values for a linear network and three commonly used nonlinear networks, illustrating low-rank structure in the penultimate layer. 
The theory for deep linear networks has led to some practical improvements, but translating it to the nonlinear setting can at times require nontrivial adjustment. 
A complete translation of the theory of emergent low-rank structures to deep nonlinear networks has yet to be developed. In our tutorial, we give the necessary background to understand important results on deep linear networks for low-rank weights during gradient training (\Cref{sec:train_dyn}) and low-rank bias with masked training (\Cref{sec:mask_train}). 
Other sections discuss how to leverage that structure algorithmically in adaptation (\Cref{sec:lora}) and training (\Cref{sec:low-rank-training}). 
Throughout the tutorial, we support both the theory and algorithmic developments with several empirical results mostly on nonlinear networks. In \Cref{sec:open} we discuss the challenges and opportunities of extending the mathematical theory to nonlinear networks.}

\smallskip 
\noindent \textbf{Structure in optimization dynamics vs. implicit regularization of the objective.} 
We will examine low-rank structures that emerge during the training of deep networks. While these structures have been explored extensively in the literature, we adopt two distinct perspectives. The first perspective focuses on the low-rank structure present throughout the iterations of the learning dynamics. The second perspective investigates how the optimization objective function itself implicitly induces such structures.

\begin{itemize}[leftmargin=*]
    \item Structure in Optimization Dynamics (\Cref{sec:LRtrain}): From this perspective, we aim to show that the dynamics of some iterative optimization algorithm applied to a particular problem have a given structure at every iteration. This structure could then be exploited throughout the entire optimization process to reduce the computation and memory cost; see \Cref{fig:svals} for an illustration.
    \item Structure at Convergence (\Cref{sec:LRobj}): From this perspective, we aim to show that while a given objective function may not have directly been designed to impose a particular structure through regularization, still the final solution may be constrained to have low-rank structure. As illustrated in \Cref{fig:masked-training-vary-iterations}, we are able to show the equivalence of the given objective to another regularized objective function, and characterize the structure of the solution at convergence of the algorithm, as opposed to at every iteration. 
\end{itemize}

\begin{figure}[t]
\centering
\begin{subfigure}{0.30\textwidth}
    \centering
    \includegraphics[width=\linewidth, trim={0.3cm, 0.4cm, 0.1cm, 0.3cm}, clip]{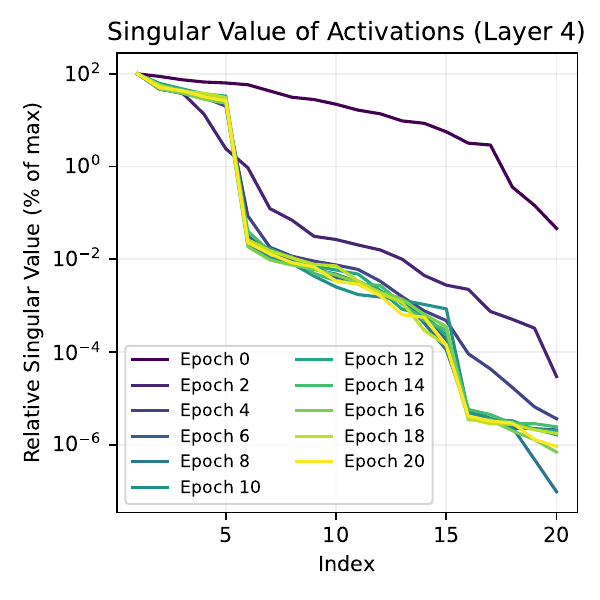}
    \caption{DLN}
    \label{fig:masked-dln-vary-iter}
\end{subfigure}
\hfill
\begin{subfigure}{0.30\textwidth}
    \centering
    \includegraphics[width=\linewidth, trim={0.3cm, 0.4cm, 0.1cm, 0.3cm}, clip]{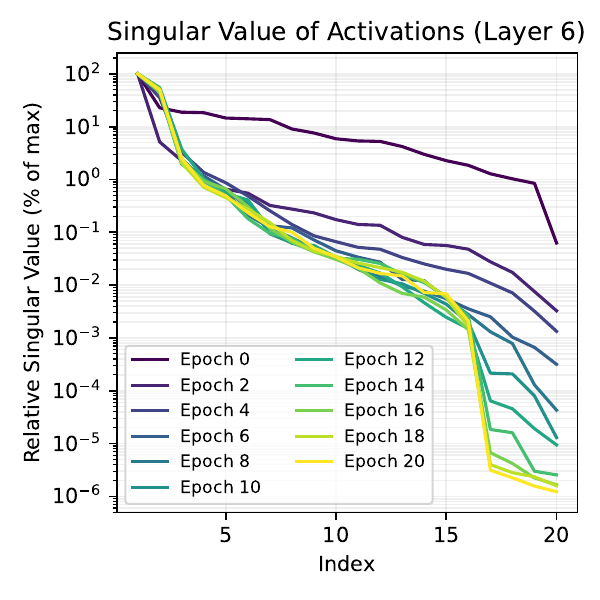}
    \caption{MLP}
    \label{fig:masked-mlp-vary-iter}
\end{subfigure}
\hfill
\begin{subfigure}{0.30\textwidth}
    \centering
    \includegraphics[width=\linewidth, trim={0.3cm, 0.4cm, 0.1cm, 0.3cm}, clip]{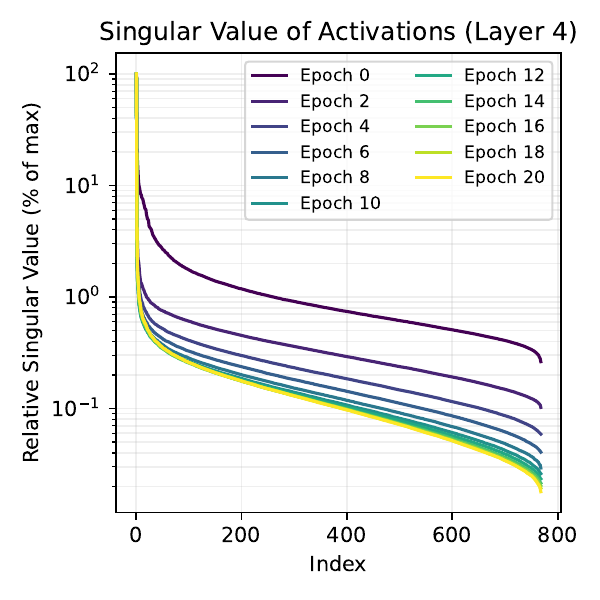}
    \caption{ResNet}
    \label{fig:masked-resnet-vary-iter}
\end{subfigure}
\vspace{-0.3cm}
\caption{\textbf{Singular values of activations of the last layer in deep architectures trained with dropout rate $60\%$, as a function of training epoch. } Deep Linear Network (DLN) and Multi-Layer Perceptron (MLP) are trained on synthetic data with MSE loss, while ResNet is trained on CIFAR-10 dataset of natural images with cross-entropy loss. Notably, the activations of the layer gradually become low-rank as masked training proceeds.} 
\label{fig:masked-training-vary-iterations}
\end{figure}

\section{Background}\label{sec:background}

\subsection{Deep Artificial Neural Networks}\label{subsec:deep_networks}

In this section, we formally introduce the basic setup of deep artificial neural networks. 
%
Mathematically, an artificial neural network can be represented as a function $f_{\bm \Theta}:\mathbb{R}^d \to \mathbb{R}^k$ that maps input data to an output space, where $\bm \Theta$ denotes the network or function parameters.
Those parameters are set through the learning or training process. 
Suppose that we have training data samples $\{(\bm x_{i}, \bm y_i)\}_{i=1}^n \subseteq \R^d \times \R^k$, where $n$ denotes the number of training samples, $\bm x_{i} \in \R^d$ denotes the $i$-th input, and $\bm y_i \in \R^k$ is an associated label or target output prediction. To learn the network parameters, one can minimize the following empirical risk over the training samples: 
\begin{align}
\label{eq:mainformulation}
        \min_{\bm \Theta} F(\bm \Theta) :=  \sum_{i=1}^{n} \mathcal{L}(f_{\bm \Theta}(\bm x_{i}), \bm y_i) + \lambda \mathcal{R}(\bm \Theta), 
\end{align}
where $\mathcal{L}:\R^k\times \R^k \to \R_+$ is the training loss function, and \revision{$\mathcal{R}: \R^{|\bm \Theta|} \to \R$} and $\lambda \ge 0$ denote the regularization function and regularization strength, respectively. Common loss functions for 
$\mathcal L$ include mean-squared error (MSE), cross-entropy (CE), and label smoothing. Much of the deep learning literature focuses on training without explicit regularization (i.e., $\lambda =0$). Nonetheless, the inherent dynamics of training can still induce implicit regularization, often giving rise to low-rank structures in network parameters, as we will discuss later.


\begin{example}[\bf Multilayer Perceptron.] A simple yet fundamental type of artificial neural network is the multilayer perceptron (MLP), which is a feedforward neural network composed of multiple layers of neurons and can be mathematically represented as a function $f_{\bm \Theta}:\mathbb{R}^d \to \mathbb{R}^k$ of the following form: 
\begin{equation}
    \label{eq:MLPf}
f_{\bm \Theta}(\bm x) = \bm W_L \sigma\left( \bm W_{L-1} \sigma \left( \cdots \sigma (\bm W_1 \bm x) \right) \right)\;,
\end{equation}
where $\bm W_l \in \mathbb{R}^{d_l\times d_{l-1}}$ denotes the weight matrix of the $l$-th layer for each $l \in [L]$ with $d_0=d$ and $d_L = k$, $\sigma(\cdot)$ is typically\footnote{Some MLP architectures have non-element-wise nonlinearities, like layer normalization and max or average pooling.} an element-wise activation function, such as ReLU, sigmoid, or tanh function \cite{parhi2023deep}, and 
$\bm \Theta = \{\bm W_l\}_{l=1}^L$ denotes the collection of all trainable network parameters. 
A deep linear network has this form with $\sigma(\cdot)$ as the identity function, i.e., 
\begin{equation}\label{eq:DLNf}
f_{\bm \Theta}(\bm x) = \bm W_L\bm W_{L-1}\ldots\bm W_1\bm x. 
\end{equation}
\end{example}

One might ask why learn a deep linear network instead of a single-layer linear mapping, given both represent linear transformations. The reason is that research shows deep, overparameterized linear networks offer two key benefits: (i) implicit low-rank bias that emerges during the training process and (ii) improved loss landscape that allows for fast convergence. 
Below in \Cref{sec:train_dyn} we will discuss the first point. Moreover, while deep linear networks are primarily of theoretical interest, they have been recognized as valuable prototypes for investigating nonlinear networks due to their simplicity and the resemblance of certain phenomena in their nonlinear counterparts. 
For example, \cite{saxe2019mathematical} demonstrated that deep linear networks exhibit a striking hierarchical progressive differentiation of structures in the internal hidden representations, resembling patterns observed in their nonlinear counterparts. To shed light on how deep networks transform inputs into outputs, \cite{wang2023understanding} demonstrated that deep linear networks exhibit layerwise feature compression and discrimination, mirroring the behavior of deep nonlinear networks. Building on these insights, recent research is actively developing tools to extend theoretical findings from linear networks to nonlinear ones, advancing our understanding of underlying principles of deep learning. 

\subsection{Low-Rank Structure in Data Matrices}\label{subsec:low_rank_bg}

This section provides background on low-rank structures in data matrices, a concept widely studied across fields such as signal processing, communications, computer vision, and medical imaging, with applications in signal approximation, direction of arrival estimation, structure-from-motion, and signal reconstruction.
Low-rank structure can take many forms, but we focus on the most basic: a linear low-dimensional subspace. When matrix columns lie on (or near) such a subspace, the matrix is exactly (or approximately) low-rank. Other low-dimensional structures, which we do not cover in this tutorial, include unions of low-rank subspaces, low-dimensional manifolds, and algebraic varieties.


A key foundational theorem 
connecting low-rank factorizations to low-rank approximations of matrices in unitarily invariant norms is as follows.
\begin{theorem}\label{thm:EY}
    Suppose that $\bm \Phi \in \R^{k\times d}$ \revision{has} a singular value decomposition (SVD) $\bm \Phi = \bm U \bm \Sigma \bm V^\top$ with $\bm \Sigma$ being a diagonal matrix of singular values $\sigma_1 \geq \sigma_2 \dots \geq \sigma_{\min(k,d)} > 0$. Consider the following optimization problem:
    \begin{equation}
        \label{eq:eckartyoung}
        \min_{\substack{\bm W_1 \in \R^{r \times d}, \bm W_2 \in \R^{k \times r}}} \| \bm W_2 \bm W_1 - \bm \Phi \|_F^2,
    \end{equation}
    where $r \le \min\{k,d\}$. Any global minimizer $(\bm W_1, \bm W_2)$ satisfies $\bm W_2 \bm W_1 = \bm U_r \bm \Sigma_r \bm V_r^\top$, where $\bm \Sigma_r$ is a diagonal matrix holding the $r$ largest magnitude singular values of $\bm \Phi$, and $\bm U_r$, $\bm V_r$ hold the corresponding singular vectors\footnote{We note for completeness that if the $r$ and $(r+1)$-st singular values are equal, then these subspaces $\bm U_r, \bm V_r$ are not unique. The singular vectors corresponding to any repeated singular values are unique up to the subspace they span.}. 
    Moreover, the resulting objective value is given by $\|\bm W_2 \bm W_1 - \bm \Phi\|_F^2 = \sum_{i=r+1}^{\min(k,d)} \sigma_i^2$.
\end{theorem}

This result was proved multiple times and is attributed to Schmidt (1907) or Eckart and Young (1936), followed by Mirsky (1960) proving that this minimizer holds for any unitarily invariant norm \cite{stewart1993early}. A key insight from decades of research is that, despite nonconvexity, many iterative optimization algorithms \revision{for solving \Cref{eq:eckartyoung} converge to its minimizer} under mild conditions. While SVD provides a solution, it is less adaptable to related but distinct problem formulations.


To connect low-rank structures to deep networks, suppose that we stack our data $\{(\bm x_i, \bm y_i) \}_{i=1}^n \subseteq \R^d \times \R^k$ into matrices: 
    $$\bm X = \begin{bmatrix} \bm x_1 & \bm x_2 & \cdots & \bm x_n \end{bmatrix} \in \R^{d\times n}, \quad \bm Y = \begin{bmatrix} \bm y_1 & \bm y_2 & \cdots & \bm y_n \end{bmatrix} \in \R^{k\times n}\;.$$ 
Using this data matrix and plugging the form of a deep linear network in \Cref{eq:DLNf} into \Cref{eq:mainformulation}, together with the $l_2$ norm as the loss function and no regularizer, yields 
\begin{equation}
\left\|\bm W_L \cdots \bm W_1 \bm X - \bm Y \right\|_F^2\;.
\label{eq:deepfactorXY}
\end{equation}
If we whiten the data by right multiplying $\bm X, \bm Y$ with $ \bm X^\top(\bm X \bm X^\top)^{-1}$, we have 
\begin{equation}
    \label{eq:deepfactor1}
\left\|\bm W_L \cdots \bm W_1 - \bm Y \bm X^\top(\bm X \bm X^\top)^{-1}\right\|_F^2 \;.
\end{equation}
Many papers in this area then consider \revision{
\begin{equation}
\left\|\bm W_L \cdots \bm W_1 - \bm \Phi \right\|_F^2\;,
\label{eq:deepfactorPhi}
\end{equation}
where \Cref{eq:deepfactor1} is a special case with} $\bm \Phi := \bm Y \bm X^\top(\bm X \bm X^\top)^{-1}$, giving the unconstrained version of \Cref{eq:eckartyoung} for $L=2$. This assumes that $\bm X \bm X^\top$ is invertible, which is usually true when $n>d$, and this is sometimes described as having an identity input matrix $\bm X$. Since the solution for $\bm W$ when minimizing $\|\bm W \bm X - \bm Y\|_F$ without constraints is simply $\bm W = \bm Y  \bm X^\top(\bm X \bm X^\top)^{-1}$, these two problems have a close connection.


\revision{Based on the above setup, we now discuss two approaches to inducing different low-rank structures, as introduced in \Cref{sec:intro}, within deep linear networks. The first approach, in the context of \Cref{eq:deepfactorXY}, explicitly enforces a low-rank constraint on the weight matrices by restricting their rank: $\mathrm{rank}(\bm W_l) \le r$ with $r < \min\{d_0,\dots,d_L\}$ for each $l \in [L]$. This, in turn, implies that both the gradients and the activations are also low-rank in deep linear networks. Specifically, taking the gradient of \Cref{eq:deepfactorXY} with respect to the weight matrix at the $l$-th layer, we get
\begin{align}
    \bm W_{l+1}^{\nextrev \top}\cdots\bm W_L^\top\left(\bm W_L\cdots\bm W_1 \bm X  - \bm Y\right) \bm W_1^\top\cdots\bm W_{l-1}^\top, 
\end{align}
which inherits a low-rank structure from the product of low-rank matrices. Likewise, the activation $\bm W_l\cdots\bm W_1 $ for each $l$ is also low-rank.
The second approach leverages an assumption on low-rank structure of the target matrix $\bm \Phi$ in the context of \Cref{eq:deepfactorPhi}. In this case, prior research has shown that gradient descent for training weights exhibits an implicit bias toward learning low-rank solutions,  even in the absence of explicit rank constraints. We will show in \Cref{thm:lop} that when $\bm \Phi$ is low-rank, each weight matrix is always updated within a low-rank invariant subspace throughout training. Then, the associated gradients and activations (defined here as the weights applied to the target) are naturally low-rank.
Adapters, on the other hand, are low-rank tunable matrices added to full-rank pre-trained weight matrices. Therefore, the overall resulting weight matrix will generally not be low-rank, but their influence on the gradient will be low-rank, and their influence on activations before nonlinear functions are applied will also be low-rank.
}    
Compared to deep linear networks, the theoretical study of low-rank structure in deep nonlinear networks is very nascent. However, there is significant empirical evidence for emergent low-rank structure in a variety of settings, as we will illustrate. 
For example, \Cref{fig:svals} shows that only a subset of the top singular values 
significantly change during weight updates from the initialization of the penultimate layer matrix, i.e., $\bm W_l^{(t)}- \bm W_l^{(0)}$; prior work supports this observation~\cite{wang2023understanding, fang2021exploring}, suggesting that deeper layers tend to have more low-rank structures than shallower ones. Moreover, the relationship among different types of low-rank structures, including weights, gradients, and activations, is more complex in nonlinear networks, as nonlinear activation functions break the simple linear propagation of low-rank structures.




\subsection{\revision{Implicit Structure in Overparameterized Optimization Dynamics}}

\label{sec:imp_struct}


Modern deep learning models are often highly underdetermined — containing far more parameters than the number of training samples — yet they demonstrate remarkable generalization to novel data. This phenomenon, known as overparameterization, raises a fundamental question: why do such models generalize well despite their capacity to fit the training data in many different ways? One prominent explanation is that the optimization dynamics used to train overparameterized models inherently favor parsimonious solutions that not only fit the data but also generalize well. This tendency, referred to as \emph{implicit bias} or \emph{implicit regularization} \cite{soudry2018implicit}, suggests that common optimization algorithms, such as gradient descent, navigate toward a subset of solutions that avoid overfitting, despite the existence of infinitely many parameter choices that could perfectly fit the training data.

To illustrate this phenomenon, in this section, we will show how the gradient dynamics of simple least-squares linear regression constrain the solution (in fact, every iterate of the optimization variable) to lie in a low-dimensional subspace. Suppose we want to predict $n$ scalar outputs $\bm y = [y_1, \ldots, y_n]^\top \in \R^{n \times 1}$ from $n$, $d$-dimensional observations $\bm X = [\bm x_1, \ldots, \bm x_n]^\top \in \R^{n \times d}$ by finding a coefficient vector $\bm w \in \R^{d}$ to solve: 
\begin{equation} \label{eq:LSbackground}
\min_{\bm w \in \R^d} \mathcal{L}(\bm w) = \frac{1}{2} \| \bm y - \bm X \bm w \|_2^2\;.
\end{equation}
Typically, one desires $n\geq d$ to ensure a unique solution to this problem. If instead we have fewer measurements than parameters $n < d$, this problem is underdetermined, and there are infinitely many solutions for $\bm w$ that will result in exactly zero error in this objective function. However, if we optimize with gradient descent, the following equation defines the dynamics of the iterative algorithm:
\begin{align}\bm w^{(t)} &= \bm w^{(t-1)} - \eta \nabla \mathcal{L}(\bm w^{(t-1)}) = \bm w^{(t-1)} - \eta \bm X^\top (\bm X \bm w^{(t-1)} - \bm y)\nonumber \\
& = \bm w^{(t-2)} - \eta \bm X^\top (\bm X \bm w^{(t-2)} - \bm y) - \eta \bm X^\top \left(\bm X \bm w^{(t-1)} - \bm y\right) \nonumber \\
& = \dots = \bm w^{(0)} - \eta \bm X^\top \sum_{i=0}^{t-1}(\bm X \bm w^{(i)} - \bm y)
\label{eq:lsgdbias}
\end{align}
Defining $\bm z^{(t-1)} = \eta \sum_{i=0}^{t-1}(\bm X \bm w^{(i)} - \bm y)$, we see that $\bm w^{(t)} = \bm w^{(0)} - \bm X^\top \bm z^{(t-1)}$. Recall that since $n<d$, $\bm X^\top$ is a tall matrix; $\bm X^\top \bm z$ is in the $n$-dimensional subspace of $\R^d$ defined by the span of the columns of $\bm X^\top$ (i.e., the row space of $\bm X$, also the orthogonal complement of the nullspace of $\bm X$). 

This observation leads us to two key points:
\begin{enumerate}
    \item \emph{Every iterate} of gradient descent must have the form
$\bm w^{(t)} = \bm w^{(0)} + \bm X^\top \bm z$ for some $\bm z \in \R^n$.
\nextrev{Equivalently, all update directions lie in $\operatorname{rowspan}(\bm X)$,
so once $\bm w^{(0)}$ is fixed, the trajectory is confined to the affine set
$\bm w^{(0)} + \operatorname{rowspan}(\bm X)$, which has intrinsic dimension
$n$ (and is contained in a linear subspace of dimension at most $n{+}1$).}
If we initialize with $\bm w^{(0)} \in \operatorname{rowspan}(\bm X)$
(e.g., $\bm w^{(0)} = \bm 0$), then every iterate lies in the rank-$n$
linear subspace $\operatorname{rowspan}(\bm X)$, which can be exploited to
significantly reduce computation when $n \ll d$ by optimizing directly in
this subspace.

    \item Gradient descent is solving an implicitly constrained problem in place of \Cref{eq:LSbackground}: 
    \begin{equation}
\min_{\bm u \in \R^d} \frac{1}{2} \| (\bm y - \bm X\bm w^{(0)}) + \bm X \bm u \|_F^2 \quad \mathrm{s.t.}\ \bm u \in \text{rowspan}(\bm X) \subset \R^d
\end{equation} or equivalently 
\begin{equation}
\min_{\bm z \in \R^n} \frac{1}{2} \| (\bm y - \bm X\bm w^{(0)}) + \bm X \bm X^\top \bm z \|_F^2 \;.
\end{equation}
This problem is no longer underdetermined, since $\bm y - \bm X\bm w^{(0)} \in \R^n$ and $\bm X \bm X^\top$ is a square $n \times n$ matrix. Again assuming that $\bm X$ has rank $n$, the unique minimizer to this convex problem is $\bm z^* = (\bm X \bm X^\top)^{-1}(\bm y - \bm X \bm w^{(0)})$, corresponding to $\bm w^* = \bm X \bm z^*$ in the original problem, and this is the solution at convergence for gradient descent. If we initialize with $\bm w^{(0)}=0$, we see that the solution is $\bm w^* = \bm X (\bm X \bm X^\top)^{-1}\bm y$, which is the Moore-Penrose pseudoinverse solution and is well-known as the minimum $l_2$ norm solution to \Cref{eq:LSbackground}.

\end{enumerate}

Although the discussion has considered the very simple problem of linear regression, we will see throughout this tutorial that many of these themes have parallels when training deep neural networks.  Specifically, in \Cref{sec:train_dyn} we detail how the dynamics of gradient descent are restricted to a low-rank subspace at every iteration when training deep, linear neural networks; that structure is leveraged in \Cref{sec:lora} and \Cref{sec:low-rank-training} to reduce computational requirements for fine-tuning and training deep networks. In \Cref{sec:var_form} we discuss how adding $l_2$ regularization on the parameters of multi-layer models is often equivalent to adding low-rank regularization on the model predictions, then we describe settings where model training implicitly adds $l_2$ regularization (and hence promotes low-rank structure) through either the implicit bias of gradient descent (\Cref{sec:low_rank_imp_reg}) or via heuristics used to train deep networks which stochastically mask the data or latent representations of the data such as Masked Autoencoders or Dropout (\Cref{sec:mask_train}). 

\section{Low-Rank Structure at Every Iteration of Gradient Dynamics}\label{sec:LRtrain}

\revision{In this section, we look at how low-rank structure emerges through the training dynamics of deep networks. Specifically, in \Cref{sec:train_dyn}, we develop the mathematical foundation based on deep linear networks for learning a low-rank target, 
with a theorem capturing the low-rank structure of weight matrices throughout the learning dynamics. Building on this, we demonstrate how the insights can inform low-rank fine-tuning and low-rank training in \Cref{sec:lora} and \Cref{sec:low-rank-training}, respectively. }

\subsection{Analysis of \revision{Weight} Training Dynamics}\label{sec:train_dyn}
 Many works study deep linear networks because they are much easier to analyze compared to their nonlinear counterparts, while still maintaining a notion of intermediate feature representations. Such works are able to demonstrate a particular implicit bias towards low-rank structure in the network weights, yielding insight into the geometry of feature representations, yet often only identify low-rank structure at the end of optimization. Instead, the result we present here discusses low-rank structure \emph{throughout} the learning process, making it amenable to exploitation by a structure-aware optimization algorithm.
To illustrate this concept, we consider a simplified deep matrix factorization setting. We train a fully connected deep linear network, as formulated in \Cref{eq:DLNf}, without nonlinear activations or biases. Additionally, we assume that the input $\bm X$ is either the identity matrix or whitened data, as described in \Cref{eq:deepfactor1}.
Specifically, we consider 
\begin{align}
\label{eq:deep_mat_factor}
    \min_{\bm \Theta}\, \underbrace{\frac{1}{2}\| \bm W_L \cdots \bm W_1 - \bm \Phi\|_F^2}_{\mathcal{L}(\bm \Theta)} + \frac{\lambda}{2} \sum_{l=1}^L \|\bm \W_l\|_F^2,  
\end{align}
where $\bm \Phi \in \mathbb{R}^{d \times d}$ is the target matrix, which can be taken as $\bm \Phi = \bm Y \bm X^\top (\bm X \bm X^\top)^{-1}$ as in \Cref{eq:deepfactor1} with $k=d$, and $\bm \Theta = \left\{\bm W_l\right\}_{l=1}^L$ are the network parameters with $\bm W_l \in \R^{d\times d}$ for each $l \in [L]$. For each weight matrix $\bm W_l$, the gradient descent updates can be written as follows: 
\begin{align}
\label{eq:deep_mat_grad}
    \bm W_l^{(t+1)} = (1-\eta\lambda)\bm W_l^{(t)} - \eta \nabla_{\bm W_l} \mathcal{L}(\bm \Theta^{(t)}),\ \forall t = 0,1,2,\dots,
\end{align}
where $\eta > 0$ is the learning rate. When we have $\lambda>0$, this is called weight decay due to the down-weighting of the weights from the previous iteration. 

We now present the main result from \cite{yaras2024compressible}, which demonstrates that the implicit bias of gradient descent in minimizing \Cref{eq:deep_mat_factor} with low-rank data arises because the entire trajectory of weight updates remains confined to a specific low-dimensional subspace. To show this, we make the following assumptions:
\begin{itemize}
    \item[\textbf{A1.}] $\bm \Phi$ is a rank $r$ matrix with $r \ll d$ (i.e., low-rank); 
    \item[\textbf{A2.}] Each initial weight matrix $\bm W_l^{(0)}$ is an $\epsilon_l$-scaled orthogonal matrix for some $\epsilon_l > 0$, i.e.,
    \begin{equation*}
        \bm W_l^{(0)} \bm W_l^{(0)\top} = \bm W_l^{(0)\top} \bm W_l^{(0)} = \epsilon_l^2 \bm I_d.
    \end{equation*}
\end{itemize}
Based on this, the following result states that the gradient updates only occur in $2r$-dimensional subspaces of the left and right singular spaces of each weight matrix $\bm W_l$.
\begin{figure}
    \centering
    \includegraphics[width=\linewidth]{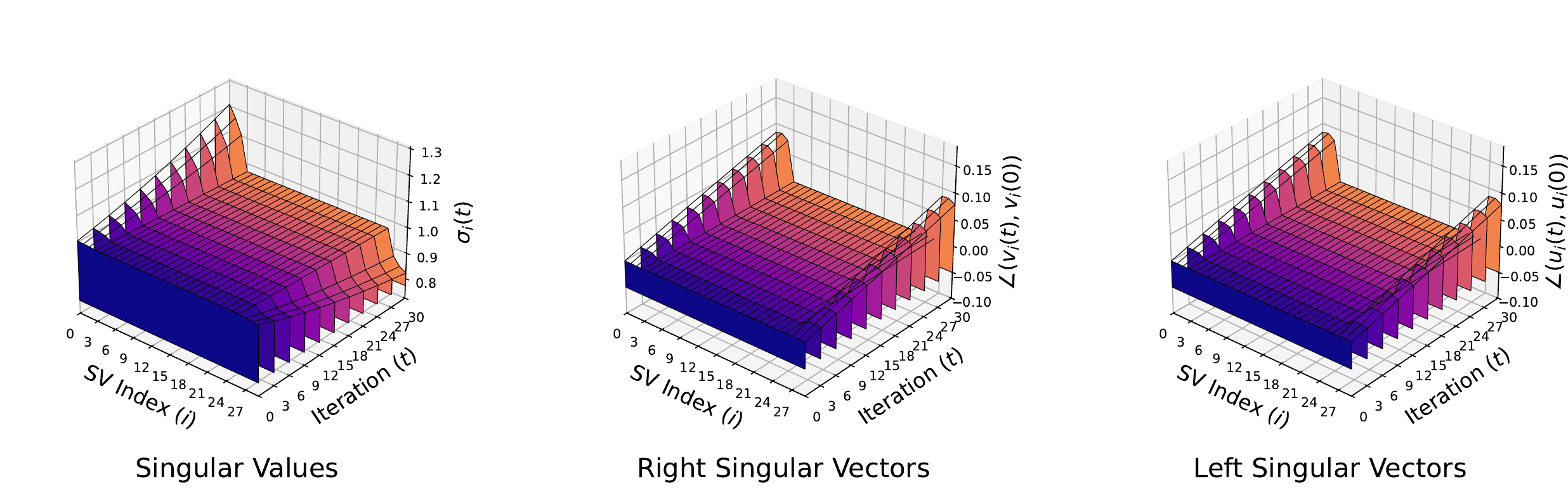}
    \caption{\textbf{Evolution of SVD of weight matrices for deep matrix factorization.}  We visualize the SVD dynamics of the first layer weight matrix of an $L=3$ layer deep matrix factorization for a random matrix with $d = 30$, $r=3$, $\epsilon_l = 1$ throughout GD without weight decay ($\lambda = 0$). \textit{Left}: Magnitude of the $i$-th singular value $\sigma_i(t)$ at iteration $t$. \textit{Middle}: Angle $\angle(\bm v_i(t), \bm v_i(0))$ between the $i$-th right singular vector at iteration $t$ and initialization. \textit{Right}: Angle $\angle(\bm u_i(t), \bm u_i(0))$ between the $i$-th left singular vector at iteration $t$ and initialization. A 2D version of this figure can be found in Supplementary \Cref{app:additional-figures}.}
    \label{fig:svd}
\end{figure}


\begin{theorem}\label{thm:lop}
Suppose we perform gradient descent on \Cref{eq:deep_mat_factor} with the iterates given in \Cref{eq:deep_mat_grad}. Under assumptions \textbf{A1} and \textbf{A2}, the SVD of weight matrix $\bm W_l^{(t)} \in \mathbb{R}^d$ can be decomposed as 
    \begin{align*}
        \bm W_l^{(t)} = \begin{bmatrix}
        \bm U_{l,1}^{(t)}  &  \bm U_{l,2}
        \end{bmatrix} \begin{bmatrix}
        \bm \Sigma_l^{(t)} & \bm{0}_{2r\times (d-2r)} \\
        \bm{0}_{(d-2r) \times 2r} & \rho_l^{(t)} \bm{I}_{(d-2r)}
        \end{bmatrix} \begin{bmatrix}
        \bm V_{l,1}^{(t)} &  \bm V_{l,2}
        \end{bmatrix}^\top \forall l \in [L], \; \forall t=0,1,2, \dots,
    \end{align*}
    where $\bm \Sigma_l^{(t)} \in \mathbb{R}^{2r \times 2r}$ is diagonal with $\bm \Sigma_l^{(0)} = \epsilon_l \bm I_{2r}$, $d-2r$ singular values are identical with 
    \begin{equation} \label{eq:rho}
        \rho_l^{(t)} = \rho_l^{(t-1)} (1 - \eta\lambda - \eta \prod_{k\neq l} \rho_k^{(t-1)2}) \approx \epsilon_l (1-\eta\lambda)^t
    \end{equation} where the orthonormal matrices $(\bm U_{l, 2})_{l=1}^L, (\bm V_{l, 2})_{l=1}^L \subset \mathcal{O}^{d \times (d-2r)}$ are independent of $t$ and satisfy $\bm V_{l+1, 2} = \bm U_{l, 2}$ for all $l \in [L-1]$, and the approximation in \Cref{eq:rho} uses $\rho_l^{(0)} = \epsilon_l$ and is more accurate as $L$ increases. \\
\end{theorem}


This theorem characterizes the decomposition of weight matrices at each iteration of gradient descent. The theory suggests that when $\epsilon_l$ is sufficiently small, the depth $L$ is large enough, and/or weight decay is applied ($\lambda>0$), a significant portion of singular values decay to zero as $t \rightarrow \infty$. The decomposition in the theorem, in the absence of weight decay ($\lambda = 0$), is visualized in \Cref{fig:svd}, where only $2r$ singular values and vectors evolve during gradient descent.

Although it applies to a simplified setting—a deep linear network trained with $l_2$ loss—it provides valuable insights into the low-rank structures throughout the learning dynamics of deep networks. 
\revision{In \Cref{sec:lora}, we will see how matrix factorizations are employed to efficiently fine-tune large language models, where the training objective is significantly more complicated. Yet even in this setting, similar phenomena in the SVD dynamics of the weights emerge when we initialize with the invariant subspace given by the bottom singular values of the gradient; see \Cref{fig:svd_lora}.}

\revision{
To understand why this structured decomposition emerges in the DLN setting, we look at the gradient descent update at initialization. Let $\bm Q = \bm W_L \cdots \bm W_{l+1}$ and $\bm R = \bm W_{l-1} \cdots \bm W_1$, so from $\epsilon_l$-scaled orthogonal initialization, we have the gradient
\begin{align}\label{eq:eps_W_A}
    \nabla_{\bm W_l} \mathcal{L}(\bm \Theta) = \bm Q^\top (\bm Q \bm W_l \bm R - \bm \Phi) \bm R^\top = \left(\prod_{k \neq l} \epsilon_k^2 \right) \bm W_l - \underbrace{\bm Q^\top \bm \Phi \bm R^\top}_{\bm A},
\end{align}
which gives the update
\begin{align}\label{eq:eps_W_A_update}
    \bm W_l^+ &= (1 - \eta \lambda) \bm W_l - \left(\prod_{k \neq l} \epsilon_k^2\right) \bm W_l + \bm A = \left(1 - \eta \lambda - \prod_{k \neq l} \epsilon_k^2 \right) \bm W_l + \bm A.
\end{align}
The $d-2r$ identical singular values that satisfy \Cref{eq:rho} correspond to a construction of $d-2r$ pairs of singular vectors $(\bm u, \bm v)$ that are simultaneously annihilated on the left and right of $\bm A$ respectively. In other words, such a pair must satisfy $\bm v \in \mathcal{N}(\bm A)$ and $\bm u \propto \bm W_l \bm v \in \mathcal{N}(\bm A^\top)$, i.e., $\bm v \in \mathcal{N}(
\bm A) \cap \mathcal{N}(\bm A^\top \bm W_l)$. Since $\bm A$ is rank $r$, each nullspace is $(d-r)$-dimensional, hence their intersection is at least $(d-2r)$-dimensional, so we should be able to find $d-2r$ such singular vector pairs as desired. These are exactly the vectors forming the columns of $\bm U_{l,2}$ and $\bm V_{l,2}$ in \Cref{thm:lop}.
}



\begin{figure}
    \centering
    \includegraphics[width=\linewidth]{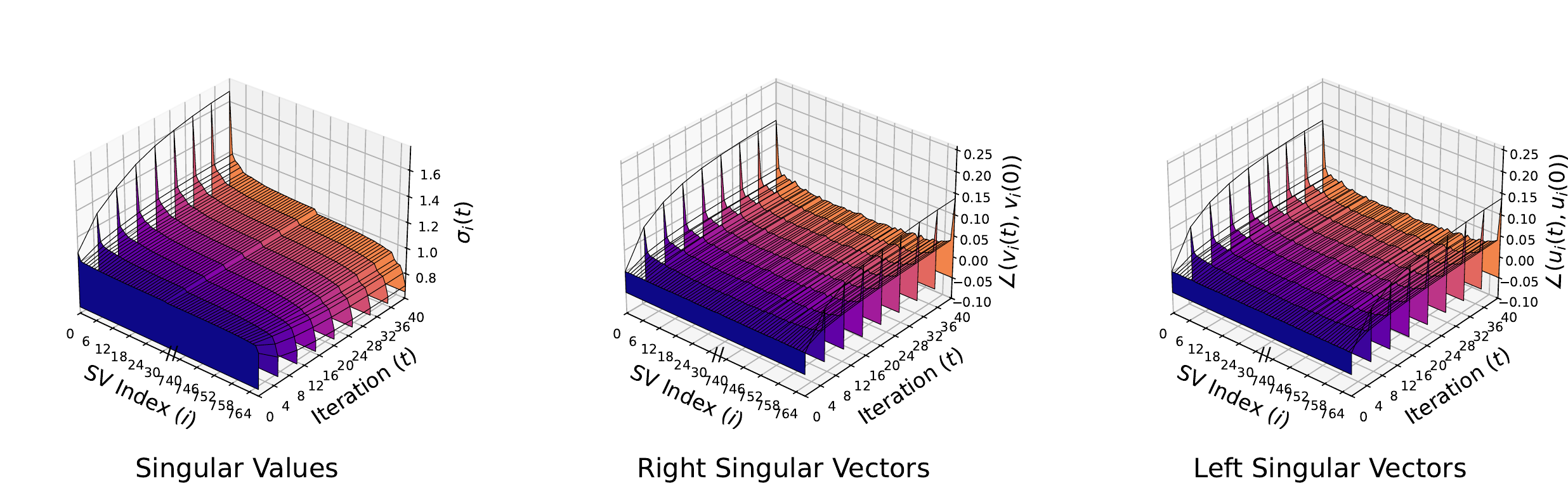}
    \caption{\textbf{Evolution of SVD of weight matrices for deep low-rank adaptation.} We visualize the SVD dynamics of an $L=3$ layer deep matrix factorization's end-to-end product employed for fine-tuning the 11th layer value matrix in BERT, with $d = 768$, $\epsilon_l = 1$ throughout Adam. \textit{Left}: Magnitude of the $i$-th singular value $\sigma_i(t)$ at iteration $t$. \textit{Middle}: Angle $\angle(\bm v_i(t), \bm v_i(0))$ between the $i$-th right singular vector at iteration $t$ and initialization. \textit{Right}: Angle $\angle(\bm u_i(t), \bm u_i(0))$ between the $i$-th left singular vector at iteration $t$ and initialization. A 2D version of this figure can be found in Supplementary \Cref{app:additional-figures}.}
    \label{fig:svd_lora}
\end{figure}

Several works laid the foundation for analyzing the dynamics of gradient descent for deep linear networks, along with its connection to implicit low-rank bias in gradient descent for low-rank factorization. \cite{saxe2014exact} pioneered this analysis, deriving closed-form solutions for singular values under gradient flow in a special initialization setting where weight singular vectors remain stationary. \cite{arora2019implicit} extended this to general initializations, establishing a stronger link between singular vector stationarity and gradient alignment. More recently, \cite{yaras2024compressible} leveraged low-rank gradients to construct orthogonal singular subspaces, distinguishing between learning-relevant and stationary components. Unlike prior work, it considers discrete-time gradient descent, arbitrary orthogonal initialization, and the role of low-rank training in compressing overparameterized deep matrix factorization. In a related but distinct direction, \cite{ji2019gradient} analyzed the alignment of deep network weights in cases where they converge to a rank-one matrix. Additionally, \cite{li2021towards} introduced a different perspective on low-rank structure in training dynamics through greedy low-rank learning, where gradient descent incrementally fits the best rank-one approximation to the residual error. Beyond these works, a broader literature explores the implicit low-rank bias of gradient descent, though not necessarily focusing on the low-dimensional structure of training dynamics—this will be discussed further in \Cref{sec:LRobj}.

\subsection{Low-rank Adapters for Fine-tuning Large-Scale Models}
\label{sec:lora}

Thus far, we have illustrated how low-rank structure emerges implicitly in the learning dynamics of the simple deep linear model. In this section, we discuss how this insight has been explicitly leveraged for parameter-efficient fine-tuning of large-scale pre-trained models. 
\revision{Fine-tuning, or adaptation, is the process of updating a pre-trained model’s parameters to adapt it to a more specific downstream task. However, optimizing each dense layer of the pre-trained model on fine-tuning data is as computationally expensive as training from scratch.
To deal with this challenge, Low-Rank Adaptation (LoRA) \cite{hu2021lora} reduces computational burden by assuming that the weight updates of overparameterized models reside in low-dimensional intrinsic subspaces.}

Concretely, let $\bm{\overline W}_{l} \in \mathbb{R}^{d \times d}$ denote a pre-trained weight matrix for a particular layer $l$, and let us introduce a trainable weight matrix $\Delta \bm{W}_l \in \mathbb{R}^{d \times d}$. During inference, we consider the effective weight $\bm{W}_l$:
\begin{align}
\bm{W}_l \;=\; \bm{\overline W}_l \;+\; \Delta \bm{W}_l,
\label{eq:full-fine-tune}
\end{align}
where $\Delta \bm{W}_l$ is the trained (or adapted) weight matrix and $\bm{\overline W}_l$ remains fixed (or frozen) throughout the entire process.
The LoRA approach assumes that the change in weights $\Delta \bm{W}_l$ also has low intrinsic rank, and considers the following effective weight instead: 
\begin{align}
\bm{W}_l \;=\; \bm{\overline W}_l \;+\; \bm{B}_{l}\,\bm{A}_{l},
\label{eq:low-rank-update}
\end{align}
where $\bm{B}_l \in  \mathbb{R}^{d \times r}$ and $\bm{A}_l \in \mathbb{R}^{r \times d}$ with $r \ll d$ are the trainable weights constrained to have at most rank $r$. For simplicity, we have limited ourselves to square weight matrices here, but in practice, this can be extended to any size weight matrix as long as each dimension is greater than $r$.

\revision{As opposed to fine-tuning all $d^2$ entries of $\Delta \bm{W}_l$, the parameterization in \Cref{eq:low-rank-update} restricts the number of learnable parameters to $2dr$, allowing us to significantly reduce memory and computation requirements if $r \ll d$. Here the adapter $\Delta \bm{W}_l$ is low-rank, but the weight matrix $\bm{\overline W}_l + \Delta \bm{W}_l$ may still be full rank. }
Next, we focus on two key aspects of LoRA which has been heavily investigated recently: how to effectively choose two key hyperparameters, the learning rate and the rank.

\subsubsection{Choosing the Learning Rate} Generally, given some learning rate $\eta > 0$, we update the LoRA parameters using some gradient-based method:
\begin{align}
   \label{eqn:lora_gd_updates} \bm{B}_l^{(t+1)} = \bm{B}_l^{(t)}  - \eta \cdot \nabla_{\bm{B}_l} \mathcal{L}(\bm{\Theta}), \quad \quad \bm{A}_l^{(t+1)}  = \bm{A}_l^{(t)} - \eta \cdot \nabla_{\bm{A}_l}\mathcal{L}(\bm{\Theta}),
\end{align}
where $\mathcal{L}(\cdot )$ denotes the loss function and $\bm \Theta$ represents the collection of network parameters. For initialization (at $t=0$), we set one of the factors to the zero matrix so that fine-tuning begins from the pre-trained model. Concretely, let $\bm a_{l, i} \in \mathbb{R}^d$ denote a column of $\bm A_1$ and let $\bm b_{l, i}^\top \in \mathbb{R}^d$ denote a row of $\bm B_l$. For all $i \in [r]$, the factors are typically initialized as follows~\cite{hu2021lora}:
\begin{align}
    &\bm{b}_{l, i}^{(0)\top} = \bm{0}_{d} \quad \text{and} \quad \bm{a}_{l, i}^{(0)}  \sim \mathcal{N}(\bm{0}, \sigma^2\bm{I}_d), \quad \text{where} \quad  \sigma^2 = \Theta\left(d^{-1}\right). 
\end{align}
While we use a Gaussian distribution for the initialization above, it is important to note that any other distribution with finite variance can be used. The variance of the factor is chosen to prevent any possible numerical instabilities. 

While this initialization was designed to enable fine-tuning from the pre-trained model, Hayou et al.~\cite{hayou2024lora} found that the asymmetry in the factors leads to ``inefficient feature learning''. They define efficient feature learning as a setup in which the norm of the change in the LoRA updates remains constant at $\Theta(1)$. However, the asymmetry  causes the norm of one factor to change much faster than the other, leading to inefficiencies (i.e., slower convergence).
We illustrate this point on the lefthand side of \Cref{fig:improving_lora}, where we fine-tune a BERT model on the STS-B dataset using a learning rate of $\eta=10^{-4}$ and rank $r=8$. Despite the test loss increasing, measured in terms of the Pearson correlation,  the norm of one factor hardly changes throughout training, while the norm of the other factor changes rapidly. To account for these observations, Hayou et al.~\cite{hayou2024lora} propose an algorithm called LoRA+ using discrepant learning rates. That is, rather than using the same learning rate for both factors as in \Cref{eqn:lora_gd_updates}, we use discrepant learning rates as follows:
\begin{align*}
\bm{B}_l^{(t+1)} = \bm{B}_l^{(t)}  - \gamma\eta \cdot \nabla_{\bm{B}_l} \mathcal{L}(\bm{\Theta}), \quad \quad \bm{A}_l^{(t+1)}  = \bm{A}_l^{(t)} - \eta \cdot \nabla_{\bm{A}_l}\mathcal{L}(\bm{\Theta}),
\end{align*}
for some constant $\gamma > 0$.
This is the main idea behind LoRA+, where they claim that we should use a larger learning rate for $\bm B_l$ such that we can account for the differences in the norms. On the righthand side of \Cref{fig:improving_lora}, we observe that we can indeed obtain faster convergence with LoRA+ using $\lambda = 2$ in the same BERT model setup. The work LoRA-Done-RITE~\cite{yen2025lora} builds upon this observation to propose an adaptive preconditioned learning rate that balances the factors. Both algorithms have been shown to improve convergence speed and achieve superior performance over the vanilla LoRA method, and we can expect future research to further address these inefficiencies and develop more efficient LoRA adapters. 

\begin{figure}[t!]
    \centering
     \begin{subfigure}[t!]{0.495\textwidth}
         \centering
         \includegraphics[width=\textwidth]{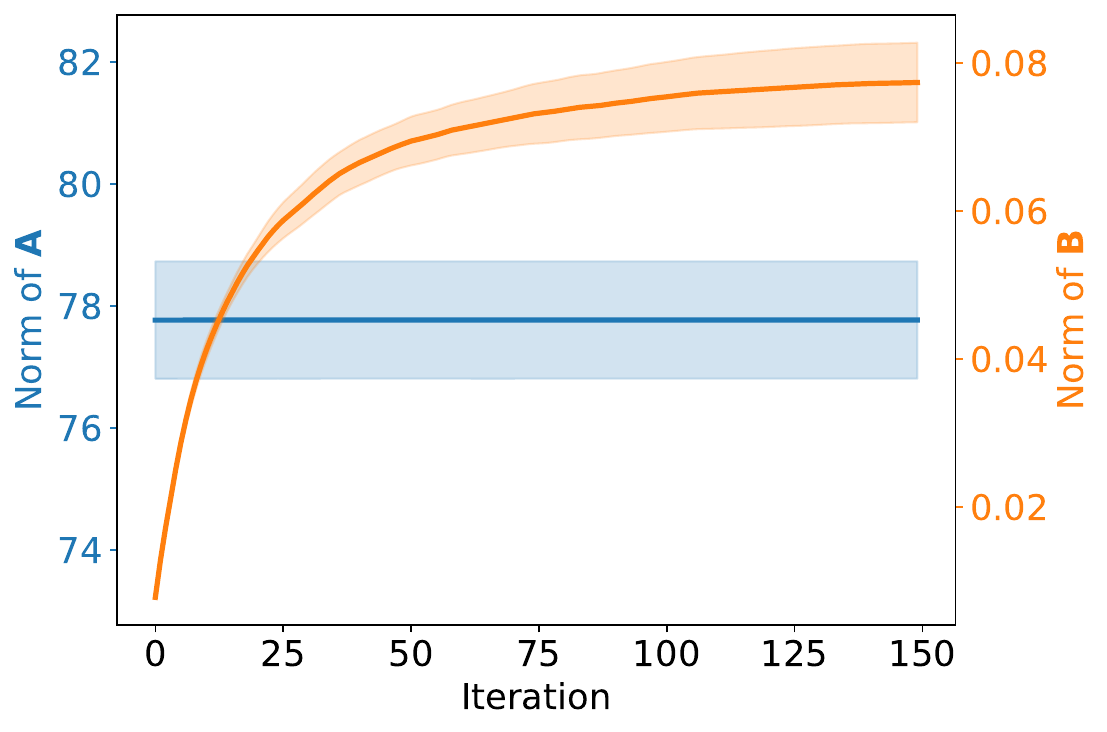}
         \caption*{Norm of LoRA Factors}
     \end{subfigure}
    \begin{subfigure}[t!]{0.495\textwidth}
         \centering
         \includegraphics[width=\textwidth]{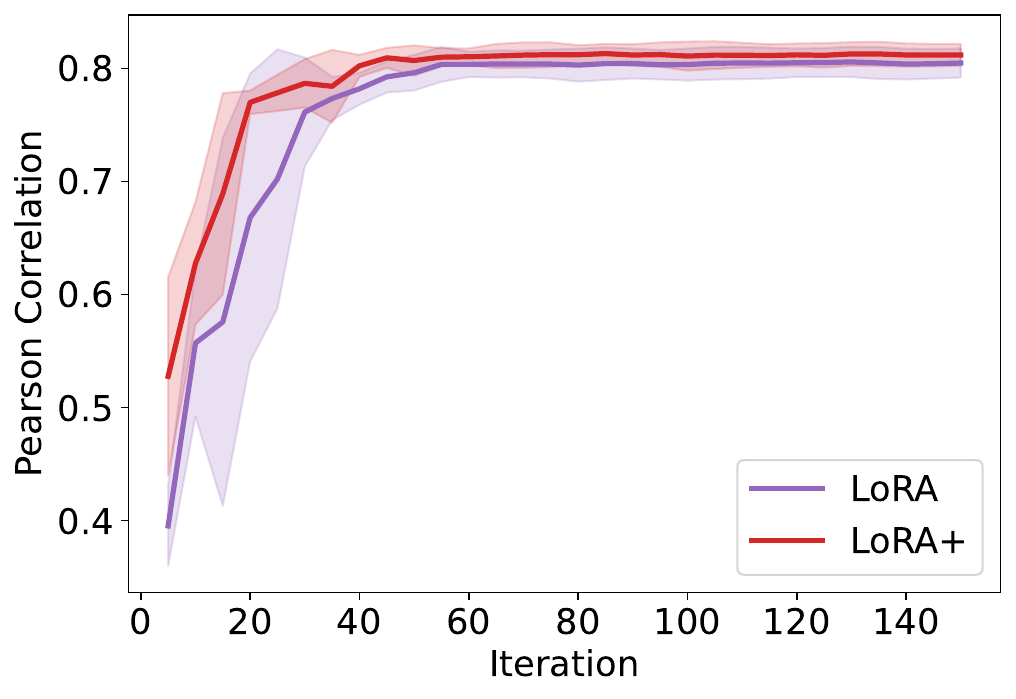}
         \caption*{LoRA+}
     \end{subfigure}\caption{\textbf{Highlighting the inefficiency of LoRA that arises from using asymmetric initializations on BERT.} Existing papers explore ways to balance learning between the two LoRA factors and justify whether one initialization is preferable to another. Left: The norm of the two factors over training iterations. 
     Right: Demonstrating that by using discrepant learning rates as in LoRA+~\cite{hayou2024lora}, we obtain faster convergence than vanilla LoRA.}
    
    \label{fig:improving_lora}
\end{figure}

\subsubsection{Choosing the Rank Allocation}\label{sec:lorarank} 
\revision{
Other than the learning rate, another critical hyperparameter is the rank of the adapters.
The choice of rank presents a  tradeoff: an insufficiently small rank can prevent the LoRA adapters from properly capturing
enough information in the training data to allow good fine-tuning performance, whereas an excessively large rank risks overfitting, slows down the fine-tuning process, and reduces the efficiency gains of LoRA.
Moreover, different layers across the network may not necessarily require the same rank. As we previously observed in \Cref{fig:svals}, deeper layers exhibit simpler structures and may therefore require a smaller rank. However, determining an optimal, layer-specific rank for each of the $L$ layers through traditional cross-validation is computationally prohibitive. 
To address this challenge, we review two works on rank allocation for LoRA.}

 \revision{\textbf{AdaLoRA} \cite{zhang2023adaptive} addresses rank allocation by dynamically distributing a predefined rank budget across each layer of the network during the fine-tuning process. The method begins with a budget schedule that evolves from an initial maximum rank for all layers (e.g., $\overline{r}L$) towards a target total budget (e.g., $rL$). To achieve this, instead of standard LoRA pairs ($\bm B_l, \bm A_l$) used in \Cref{eq:low-rank-update}, AdaLoRA decomposes each adapter $\Delta \bm W_l$ into triplets: $\Delta \bm W_l= \bm P_l \bm \Lambda_l \bm Q_l$. Here, $\bm \Lambda_l$ is a diagonal matrix whose entries can be pruned to zero for reducing the rank allocated to layer $l$. The decision of which entries to prune is based on a scoring function applied across all layers in each fine-tuning epoch, where different scoring functions have been studied in the paper \cite{zhang2023adaptive}. To ensure that zeroing these entries meaningfully reduces the rank $\Delta \bm W_l$, similar to singular value decomposition, we would require $\bm P_l$ and $\bm Q_l$ to be orthogonal matrices. To achieve this, AdaLoRA adds an orthogonality regularization term to both $\bm P_l$ and $\bm Q_l$ (i.e., $\|\bm P \bm P^\top - \bm I\|+\|\bm Q\bm Q^\top - \bm I\|$) to the objective function during fine-tuning. Despite the added computational overhead from these mechanisms, experiments in \cite{zhang2023adaptive} demonstrate that AdaLoRA achieves improved accuracy compared to fixed-rank allocation when fine-tuning across various benchmark datasets. While effective, AdaLoRA is largely heuristic and lacks a principled explanation of its working mechanism. }

\begin{figure*}[t!]
\begin{center}
\includegraphics[width=0.32\linewidth]{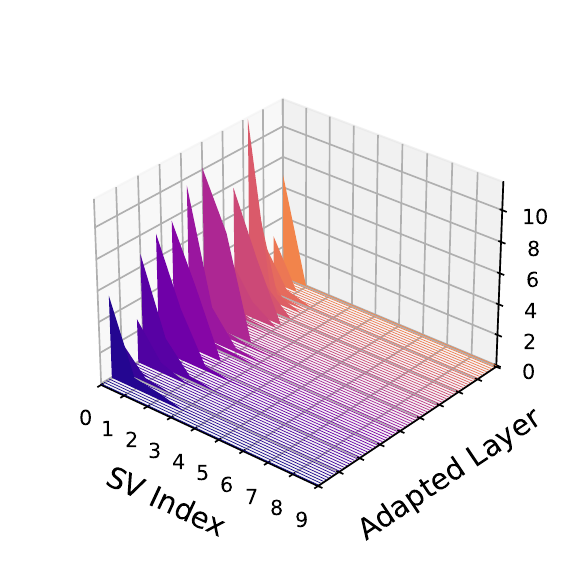}
\includegraphics[width=0.32\linewidth]{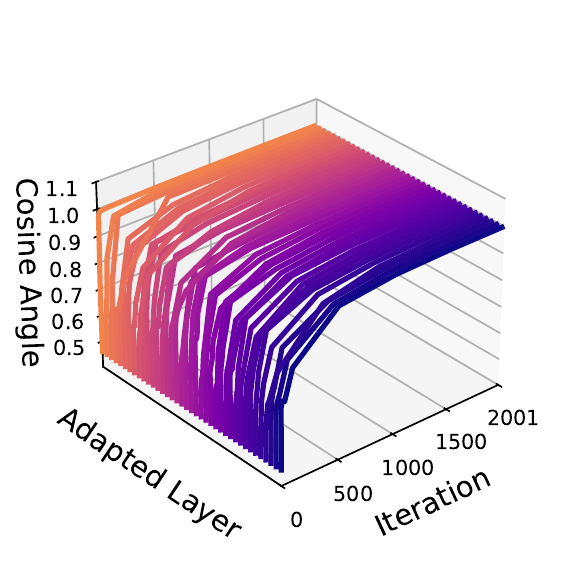}
\includegraphics[width=0.32\linewidth]{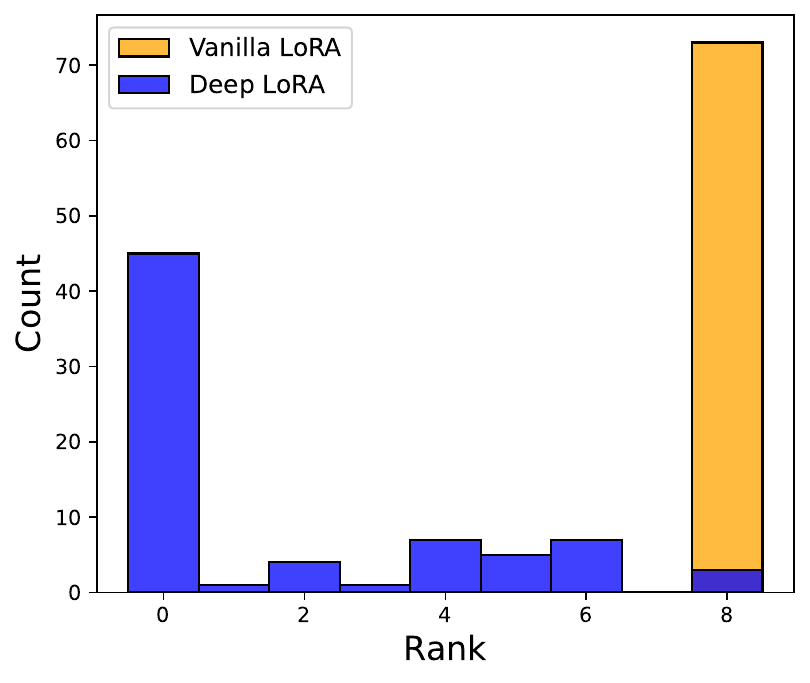}
\end{center}
\caption{\textbf{Invariant low-dimensional subspaces in deep overparameterized adaptation of language models.} Fine-tuning BERT with deep overparameterized adaptation on the STS-B dataset. \textit{Left}: \textbf{Singular value} spectra across all adapted layers at the end of fine-tuning. \textit{Middle}: \textbf{Alignment of subspaces} formed by the top 8 right singular vectors between the current adapted weights and the final adapted weights throughout training. \textit{Right}: \textbf{Deep LoRA finds lower ranked adapters} as opposed to LoRA, allowing it to have better test loss in the sample-starved regime. \nextrev{A 2D figure with log y-axis can be found in Supplementary \Cref{app:additional-figures}.}}
\label{fig:low_rank_subspace}
\end{figure*}

\revision{In contrast, \textbf{Deep LoRA} \cite{yaras2024compressible} offers a theory-inspired approach, 
where we take what we know about low-rank bias of deep linear factorizations \cite{arora2019implicit} and apply it to LoRA. 
We first show empirical evidence in \Cref{fig:low_rank_subspace} that demonstrates this low-rank bias in fine-tuning with overparameterization: when a deep nonlinear network is fine-tuned with a deep overparameterized LoRA (using three $d \times d$ factors), the converged weights $\{\Delta \bm W_l\}_{l=1}^L$ exhibit very low rank (\Cref{fig:low_rank_subspace}, left panel). Moreover, the learning dynamics for each weight matrix approximately remain within the same invariant subspace throughout training (\Cref{fig:low_rank_subspace}, middle panel), which aligns with our findings for deep linear networks in \Cref{sec:train_dyn}. 
These observations suggest that full-dimensional adaptation, or full fine-tuning, is highly compressible, which explains why LoRA is effective in practice.}

\revision{Deep LoRA leverages this finding by applying a compression method to (\emph{i}) exploit the benefits of depth in overparameterization for adaptive rank allocation and (\emph{ii}) mitigate computational overhead by reducing the width of intermediate layers. These lead to an update to each weight matrix as follows:
\begin{equation*}
    \bm W_l = \bm{\overline W}_l + \underbrace{\bm{C}_l\, \bm{B}_{l}\,\bm{A}_{l}}_{ =:\Delta \bm W_l },
\end{equation*}
where $\bm A_l \in \mathbb{R}^{r \times d}$, $\bm B_l \in \mathbb{R}^{r \times r}$, and $\bm C_l \in \mathbb{R}^{d \times r}$. 
Using the invariant gradient subspace discussed after \Cref{thm:lop} as an initialization, Deep LoRA trains this compressed depth-3 adapter, whose implicit bias is able to implicitly select the right rank per layer. Deep LoRA is therefore
particularly effective in data-limited fine-tuning regimes. In (\Cref{fig:low_rank_subspace}, right panel), we show a histogram of the numerical rank for each layer of BERT fine-tuned with LoRA and Deep LoRA. While LoRA saturates all layers at $r=8$, Deep LoRA's implicit bias allows the numerical rank of the update matrix to be much smaller -- even zero for the majority of the layers.}

\revision{Although the Deep LoRA approach is inspired by the theory for deep linear networks in \Cref{sec:train_dyn}, a key distinction from the deep matrix factorization setting (in \Cref{sec:train_dyn}) is that Deep LoRA's target is not strictly low-rank, meaning updates are not confined to a fixed invariant subspace. To address this, similar to LoRA+ \cite{hayou2024lora}, Deep LoRA employs a discrepant learning rate: the outer factors $\bm C_l$ and $\bm A_l$ are updated more slowly 
than the inner factor $\bm B_l$. Empirically, as shown in \Cref{fig:low_rank_subspace} (right panel) for a BERT model fine-tuned with Deep LoRA (e.g., rank $r=8$, depth $3$ adapters), Deep LoRA implicitly finds adapters with varying, lower effective ranks than vanilla LoRA. This explains its superior test accuracy in data-limited fine-tuning scenarios \cite{yaras2024compressible}.
}

\subsection{Low-Rank \revision{Gradients} for Efficient Training}\label{sec:low-rank-training}


More recently, the implicit low-rank structures in learning dynamics \cite{yaras2024compressible} (\Cref{sec:train_dyn}), coupled with the success of LoRA (\Cref{sec:lora}), have inspired researchers to explicitly explore low-rank structures for training large-scale deep networks. Instead of introducing additional low-rank adapters and hence not altering the training dynamics \cite{lialin2024relora,zhao2024galore,jaiswal2024WeLore}, these approaches have the potential to significantly reduce the computational and memory requirements of training large-scale neural networks by approximating the gradient updates of weight matrices with low-rank factorizations.


However, the optimal weights of nonlinear neural networks are often \emph{not} inherently low-rank, as evidenced by the fact that LoRA may not always match the performance of full-rank fine-tuning and often benefits from an initial full-rank training phase to effectively leverage the low-rank subspace \cite{hu2021lora}. Nonetheless, several recent works \cite{yaras2024compressible,zhao2024galore,jaiswal2024WeLore,refael2025adarankgrad} showed that
gradient updates of the weights are approximately low-rank, where the gradients approximately reside in slowly changing low-rank subspace across the learning dynamics. For deep linear models, the low-rank structure of gradient updates follows directly from our training dynamics analysis in \Cref{sec:train_dyn} and the findings in \cite{yaras2024compressible,kwon2023efficient}, where the gradient captures the difference between two consecutive weight updates within the same low-rank subspace.
More recent work \cite{zhao2024galore} justifies the low-rank gradient for reversible nonlinear networks, showing the gradient becomes low-rank during training when it follows certain parametric forms.


Specifically, if we consider the weight matrix of any given layer (e.g., $l$th layer) in a neural network, then slow changing low-rank gradient implies that we can approximate the weight $\bm W_{KT}$ at the $KT$-th iteration by
\begin{align}\label{eqn:gradient-update}
    \bm W_{KT} = \bm W_0 +  \sum_{k=1}^K \Delta \bm W_{T_k}, \; \text{where}\; \Delta \bm W_{T_k} = \sum_{t_k=1}^T \bm G_{t_k},
\end{align}
where $\bm W_0$ is the weight initialization, $T$ denotes the window size of each subspace, and $\bm G_{t_k}$ denotes the $t$-th gradient update in the $k$-th subspace $\bm U_k$. In other words, the accumulated gradient $\Delta \bm W_{T_k}$ lies within a low-dimensional subspace $\bm U_k$, where the subspace $\bm U_k$ dynamically changes with respect to $k$ \cite{lialin2024relora,zhao2024galore}, see \Cref{fig:low-rank-gradient} for an illustration. This phenomenon can yield significant benefits: by restricting updates to a subspace of reduced dimension, one can potentially reduce both computational cost and memory usage, since operations in higher-rank dimensions become negligible \cite{lialin2024relora,zhao2024galore}.

 \begin{wrapfigure}{r}{0.36\textwidth}
        \centering
        \includegraphics[width=0.35\textwidth]{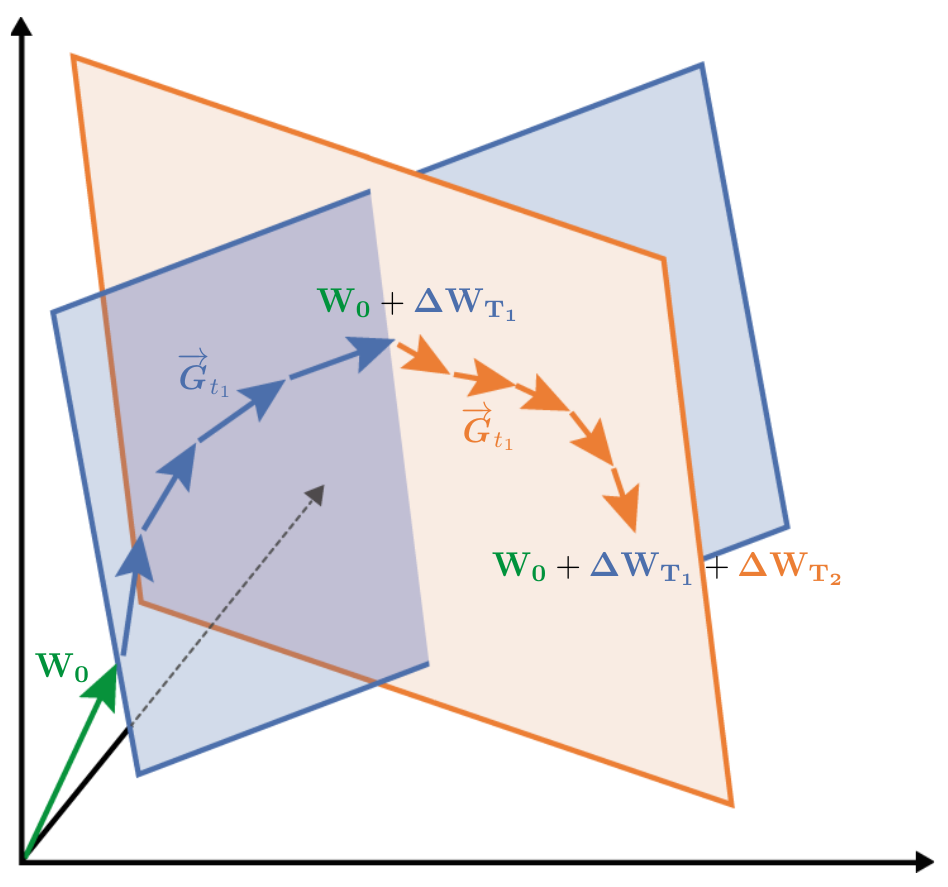}
\caption{\bf An illustration of learning through low-rank subspaces using GaLore \cite{zhao2024galore}.}
    \label{fig:low-rank-gradient}
\end{wrapfigure}
\paragraph{Exploring Low-rank Gradient for Low-Rank Training} In practice, many recent works leverage low-rank structures in the gradient to reduce memory usage. Specifically, based upon \Cref{eqn:gradient-update}, Gradient Low-Rank Projection (GaLore) method \cite{zhao2024galore} that leverages the gradient low-rank structures in \Cref{eqn:gradient-update} through subspace projection 
\revision{
\begin{align}
    \tilde{\bm G}_{t_k} = \bm P_{k} \rho ( \bm P_{k}^\top \bm G_{t_k} \bm Q_{k})  \bm Q_{k}^\top,
\end{align}
}
to better capture the information of the full-rank gradient $\bm G_{t_k}$. Here, $\bm P_k \in \mathbb R^{m \times r}$ and $\bm Q_k \in \mathbb R^{n \times r}$ are the subspace projection matrices of the original gradient $\bm G_{t_k} \in \mathbb R^{m \times n}$, and $\rho(\cdot)$ is an operator on gradient that flexibly encompasses a variety of gradient-based methods such as momentum or Adam. To compensate for the dynamically changing subspaces, GaLore updates $\bm P_k$ periodically by computing the full SVD of $\bm G_{t_k=1}$ at the beginning of every $T$ epoch. 
We present the algorithm applying GaLore to Adam in \Cref{alg:low_rank_adam} in the appendices, where $\rho(\cdot)$ is instantiated as the Adam update rule. As shown in \cite{zhao2024galore}, GaLore significantly improved memory efficiency when training LLMs, reducing optimizer state memory usage by up to 65.5\%. Inspired by \Cref{eqn:gradient-update}, there has also been another ReLoRA method \cite{lialin2024relora} that uses randomly initialized LoRA pairs $\bm B_k\bm A_k$ to approximate $\Delta \bm W_{T_k}$ and merge them with the full weights $\bm W_{0}$ every $T$ iterations. Compared to ReLoRA, GaLore achieves better performance with similar memory cost, but at the cost of extra computation time due to extra projection operators and computing the initial full rank gradient $\bm G_{t_k=1}$ for each $k=1,\cdots,K$. Moreover, when ReLoRA uses the full-rank gradient as initializations, it can be shown that it is equivalent to GaLore; more detailed proof is shown in \Cref{app:low-rank-training}.


\begin{figure}[t]
\includegraphics[width=\linewidth]{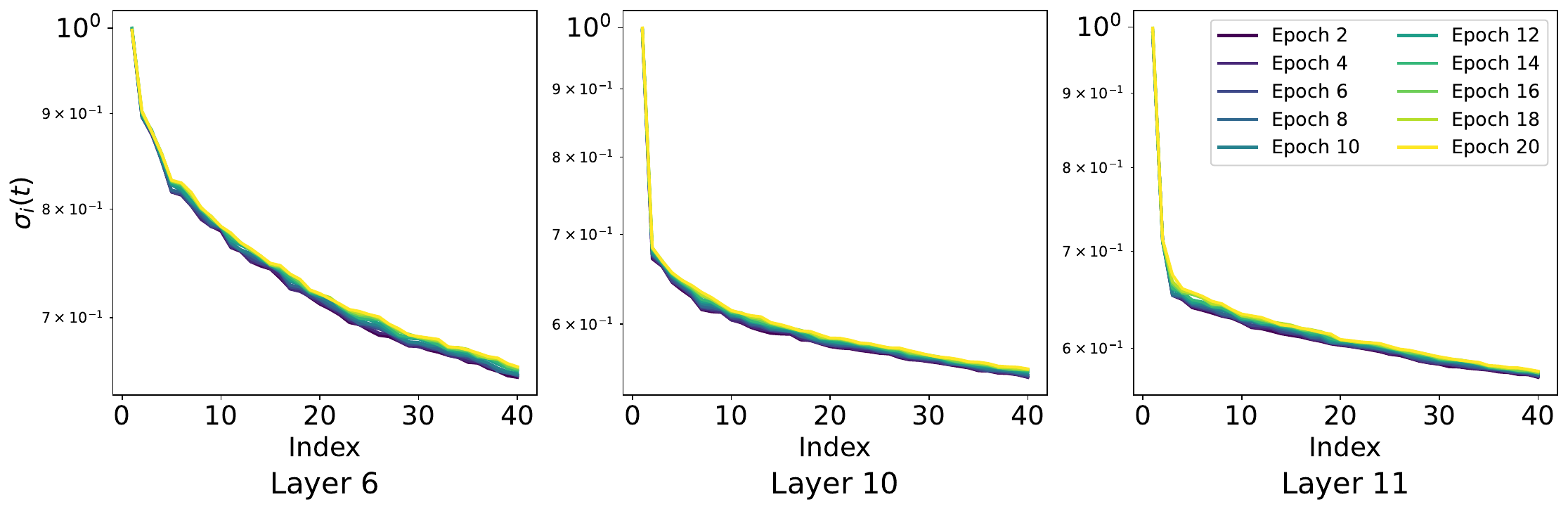}
     \caption{\textbf{Singular values across layers of the ViT of the projection weight matrix.} While the singular values of the shallower layers do not clearly exhibit a low-rank structure, the deeper layers certainly do. }
    \label{fig:vitfc2}
\end{figure}

\paragraph{Improving Low-Rank Training via Adaptive Rank Allocation across Different Layers and Components} While GaLore offers remarkable improvements in memory efficiency, its uniform treatment of different layers and usage of the same training window can be suboptimal for large-scale deep learning models. This phenomenon is illustrated in \Cref{fig:vitfc2}, which shows the singular values of different MLP layers in ViT models throughout training. As observed, deeper layers tend to exhibit more pronounced low-rank structure compared to shallower layers. A similar trend is evident in the value weight matrices of the self-attention components; see \Cref{app:low-rank-training} for details.

Therefore, the follow-up results, such as Weight Low-Rank Projection (WeLore) \cite{jaiswal2024WeLore}, reveal the significant differences in effective ranks across different layers and training epochs in LLMs. AdaRankGrad \cite{refael2025adarankgrad} demonstrates that layer gradient ranks gradually decrease as training converges, asymptotically approaching one. Leveraging this, it adaptively reduces gradient rank during Adam optimization using an efficient online low-rank projection update. The method dynamically adjusts each layer’s projection dimension based on training progress, exploiting the gradual reduction in gradient dimensionality to perform updates in a lower-dimensional space, thereby reducing memory usage. Projection matrices are updated based on a convergence criterion, ensuring timely adjustments and preventing premature or delayed subspace transitions, leading to faster convergence. Compared to GaLore, AdaRankGrad further improves memory efficiency for pretraining of LLMs \cite{refael2025adarankgrad}. 

Moreover, different components of Transformer models exhibit distinct low-rank structures that can be leveraged for compression. For instance, the recent DeepSeek V3 model \cite{liu2024deepseek} achieves significant memory savings during training by only factorizing the query matrix into a low-rank form within the multi-head attention mechanism, all while maintaining strong performance. Finally, in a broader perspective, the low-rank structure is one example of a more general class of compressive structures that have gained recent attention. We refer interested readers to in the appendices for a more detailed discussion.

\subsection{Implications of Low-Rank Structures}


\subsubsection{Low-Rank Compression of Neural Networks}\label{subsec:compress}

\revision{Low-rank structures that emerge during neural network training have long been regarded as an effective tool for compressed network design \cite{idelbayev2020low} or post-training compression \cite{cheng2018model}. A common pipeline is to first train a full model, then approximate its weight matrices using low-rank decomposition, such as SVD and Tucker decomposition, and finally fine-tune the compressed model to recover accuracy \cite{cheng2018model}.  For instance, \cite{denton2014exploiting} demonstrated that convolutional filters in CNNs often exhibit significant redundancy and can be well-approximated by low-rank tensor factorizations. \cite{jaiswal2024WeLore} highlighted the consequential relationship between low-rank weights and low-rank gradients \cite{zhao2024galore}, which can be leveraged for block-adaptive transformer compression without need for fine-tuning. 
Complementing these weight-centric approaches, \cite{yu2023compressing} showed that transformer activations, rather than weights, exhibit stronger low-rank structure, motivating feature-based compression strategies. }

\revision{
\subsubsection{Relation to Low‑Rank INRs and PINNs.}
\label{sec:INRPINN}
Beyond using low‑rank updates to reduce the cost of training large networks, an adjacent thread in scientific machine learning, including \emph{implicit neural representations (INRs)} and \emph{ physics‑informed neural networks (PINNs)}, explicitly impose low‑rank structure onto the representation itself. In low‑rank PINNs, 
a hypernetwork first learns a shared low‑rank basis, after which only a small set of coefficients are adapted per new PDE instance; this meta‑learning approach has been shown to mitigate PINN failure modes and drastically cut adaptation cost~\cite{cho2023hyperlrpinn}. Building on this architecture, \emph{Sparse Physics‑Informed Backpropagation} performs the backward pass on a reduced surrogate 
derived from the low‑rank representation, yielding substantial backpropagation speedups while retaining accuracy when mapped back to the full model~\cite{cho2025fastlrnr}. These ideas resonate with our gradient‑subspace viewpoint: both restrict learning to a low‑dimensional manifold, but the INR/PINN route enforces the manifold at the architectural level (often connected to classical reduced‑order modeling such as EIM/DEIM‑style sampling) while our discussion emphasizes low‑rank geometry of the update. 
}

\section{Low-Rank Structure at the Convergence of Gradient Methods}
\label{sec:LRobj}

As we have seen above, low-rank structure can arise due to the dynamics of a particular optimization algorithm that enforces (or exploits in the case of LoRA methods) the low-rank structure throughout the optimization process (i.e., at every iteration of the optimization algorithm the low-rank structure is present).   However, low-rank structure also arises through other means in a variety of deep learning settings.  
In particular, in this section we explore various conditions and algorithms where 
the final solution the algorithm converges to has low-rank structure.  If we recall the simple example of overparameterized linear regression from \Cref{sec:imp_struct}, we observed that 
of the infinitely many possible choices of parameters that perfectly fit the data, gradient descent would implicitly find the parameters that fit the data with the minimum $l_2$ norm.
%
%
As a result, $l_2$ regularization has been implicitly added to the problem due to the dynamics of gradient descent, and in this section, we will see a variety of additional settings where 
\revision{the network training process induces implicit $l_2$ regularization on the activations of the network.}
At first, this does not present an immediate link to low-rank structure, but as we discuss in the next section, in many models consisting of sequential application of linear operators, applying $l_2$ regularization to the model parameters is equivalent to adding low-rank regularization to the \revision{output activation of the final layer of the network, e.g., the} overall prediction of the model. \revision{Finally, it is additionally worth noting that forms of regularization beyond the $l_2$ norm on the model parameters can be induced by other optimization algorithms: for instance, $\ell_1$-norm implicit bias can be induced by coordinate descent. We refer the readers to the work of \cite{pmlr-v80-gunasekar18a} and references therein. }

\subsection{\revision{Low-Rank Variational Forms}} 

\label{sec:var_form}

A key link that arises in the emergence of low-rank structure is the general idea that in many settings adding $l_2$ regularization on the weight parameters can result in adding low-rank inducing regularization on the predictions of the model on the training set.  In particular, recall the deep matrix factorization objective \Cref{eq:deep_mat_factor} in \Cref{sec:train_dyn},
%
%
where we have observed that optimizing this objective with gradient descent will result in low-rank parameter matrices when $l_2$ regularization is used (i.e., $\lambda > 0$), commonly referred to as \emph{weight decay} in the context of deep learning, and a natural question is whether this behavior as a idiosyncratic to the gradient descent dynamics or a more general phenomenon.  As we will see below, low-rank solutions are, in fact, promoted by the use of $l_2$ regularization on the model weights due to variational definitions of Schatten-$p$ matrix norms and quasi-norms.

Given a matrix $\bm M \in \R^{m \times n}$ and a value of $p>0$ the Schatten-$p$ matrix norm  (for $p \geq 1$) or quasi-norm (for $0 < p < 1$) is notated as $\| \bm M\|_p$ and defined:
\begin{equation}
\| \bm M\|_p = \left( \sum_{i=1}^{\min\{m,n\}} \sigma_i(\bm M)^p \right)^{\tfrac{1}{p}}
\end{equation}
where $\sigma_i(\bm M)$ denotes the \nextrev{$i$-th largest singular value} of $\bm M$. This includes many well-known matrix norms, such as the Frobenius norm $(p=2)$, the $l_2$ induced operator norm (the largest singular value of $\bm M$ for $p=\infty$), and the nuclear norm $(p=1)$.  The nuclear norm in particular has received considerable attention in the literature on low-rank matrix recovery, as it is the tightest convex approximation of the rank of a matrix, which allows it to be used as a convex regularizer to induce low-rank solutions.  Moreover, if one considers Schatten-$p$ quasi-norms with $0 < p <1$ (so called because they no longer satisfy the triangle inequality) one achieves an increasingly tighter approximation of the rank of a matrix in the limit as $p \rightarrow 0$, with the downside being that the quasi-norm is no longer convex for $0 < p < 1$.  

To show the connection between Schatten-$p$ (quasi)norms and the matrix factorization objective in \Cref{eq:deep_mat_factor} we will rely on a variational definition of $\| \bm M \|_p$.  In particular, the following result can be shown:
%
\begin{corollary} (Adapted from \cite{shang2020unified}) Given a matrix $\bm M \in \R^{m \times n}$ with $\text{rank}(\bm M) \leq d$ and $L \geq 2$, define matrices $\bm W_L \in \R^{m \times d}$, $\bm W_1 \in \R^{d \times n}$, and $\bm W_l \in \R^{d \times d}$ for $l = 2,\ldots,L-1$.  Then, we have that
\begin{equation}\label{eq:variational-shatten-p}
\frac{2}{L} \left(\| \bm M \|_{\tfrac{2}{L}}\right)^{\tfrac{2}{L}} = \min_{\bm W_L, \ldots, \bm W_1}  \frac{1}{2} \sum_{l=1}^L \| \bm \W_l \|_F^2 \ \ \text{s.t.} \ \ \bm W_L \bm W_{L-1} \cdots \bm W_2 \bm W_1 = \bm M.
\end{equation}   
\label{cor:variational-shatten-p}
\end{corollary}
The result in \Cref{eq:variational-shatten-p} implies that if we minimize any matrix factorization model with Frobenius regularization on the matrix factors, then we are effectively minimizing a problem with Schatten-$\tfrac{2}{L}$ regularization on the product of the $\bm W$ matrices (potentially also with rank constraints on the product if $d < \min\{m,n\})$.  

\revision{Concretely, consider the following three alternative problems:
\begin{align}
  & \min_{\bm W_L, \ldots, \bm W_1}  f(\bm W_L \bm W_{L-1} \cdots \bm W_2 \bm W_1) + \frac{\lambda}{2}  \sum_{l=1}^L \| \bm W_l \|_F^2 \label{eq:gen_W} \\
     & \min_{\bm M} f(\bm M) + \frac{2 \lambda}{L} \left(\|\bm M\|_{\tfrac{2}{L}} \right)^{\tfrac{2}{L}} \ \ \text{s.t.} \ \ \text{rank}(\bm M) \leq d. \label{eq:gen_M} \\
  & \min_{\bm W_L, \ldots, \bm W_1, \bm M}  f(\bm M) + \frac{\lambda}{2} \sum_{l=1}^L \| \bm W_l \|_F^2  \ \ \text{s.t.} \ \ \bm W_L  \bm W_{L-1} \cdots \bm W_2 \bm W_1 = \bm M \label{eq:gen_WM} 
\end{align} Using the variational identity in \Cref{eq:variational-shatten-p}, the problem in \Cref{eq:gen_W} can be rewritten by introducing \(\bm M=\prod_{\ell} \bm W_\ell\) and minimizing over \(\{\bm W_\ell\}\) for each fixed \(M\):
\[
\min_{\bm W_L, \ldots, \bm W_1}  f(\bm W_L \bm W_{L-1} \cdots \bm W_2 \bm W_1) + \frac{\lambda}{2}  \sum_{l=1}^L \| \bm W_l \|_F^2
\;\equiv\;
\min_{\bm M}\Big\{\,f(\bm M)+\frac{2\lambda}{L}\|\bm M\|_{2/L}^{2/L}\,\Big\},
\]
provided \(d \ge \operatorname{rank}(\bm M)\). This yields problem \Cref{eq:gen_M}; inserting \(\bm M\) as an explicit variable with the constraint \(\bm M=\prod_{\ell}\bm W_\ell\) gives \Cref{eq:gen_WM}. Thus, these three problems have the same optimal value; moreover, if \(\{\bm W_\ell^\star\}\) is a global minimizer of \Cref{eq:gen_W}, then \(\bm M^\star=\prod_{\ell}\bm W_\ell^\star\) is a global minimizer of \Cref{eq:gen_M}, and conversely any global minimizer \(\bm M^\star\) of \Cref{eq:gen_M} admits a balanced 
factorization \(\{\bm W_\ell^\star\}\) that attains the inner minimum in \Cref{cor:variational-shatten-p}, making \((\bm M^\star,\{\bm W_\ell^\star\})\) optimal for \Cref{eq:gen_WM}.}

As an example, if we return to the deep matrix factorization problem in \Cref{eq:deep_mat_factor} we have that the problem is equivalent to finding an approximation of $\bm \Phi$ by a matrix with small singular values (and hence low-rank):
\begin{align}
\min_{\bm W_L, \ldots, \bm W_1} & \ \ \ \frac{1}{2} \| \bm W_L \cdots \bm W_1 - \bm \Phi\|_F^2 + \frac{\lambda}{2} \sum_{l=1}^L \| \bm W_l\|_F^2 \iff \\ 
\nonumber \min_{\bm M : \ \text{rank}(\bm M) \leq d} & \left\{ \frac{1}{2} \| \bm M - \bm \Phi \|_F^2 + \frac{2 \lambda}{L} \left( \| \bm M \|_{\tfrac{2}{L}} \right)^{\tfrac{2}{L}} = \frac{1}{2} \| \bm M - \bm \Phi \|_F^2 + \frac{2 \lambda}{L} \sum_{i=1}^d \sigma_i(\bm M)^{\tfrac{2}{L}} \right\}.
\end{align}
\nextrev{Here, $\iff$ denotes equivalence in the sense that the problems admit corresponding minimizers,
with $\bm M^* = \prod_{\ell=1}^L \bm W_\ell^*$.} If we take the simple example of $L=2$ (and with $d$ sufficiently large to ignore the rank constraint), then the regularization on $\bm M$ becomes the nuclear norm, and the closed form solution for the optimal $\bm M$ is well-known to be the soft thresholding of the singular values of $\bm \Phi$ -- i.e., if we let $\bm U \bm S \bm V^\top = \bm \Phi$ be a singular value decomposition of $\bm \Phi$, then the optimal value is $\bm M = \bm U (\bm S - \lambda \bm I)_+ \bm V^\top$.  As a result, if $\bm \Phi$ has any singular values smaller than $\lambda$, the recovered solution will also be low rank.  Likewise, for deep matrix factorizations $(L > 2)$ the singular values of the final recovered $\bm M$ matrix will also be shrunk more aggressively as the depth of the regularization ($L$) increases.

\subsection{Low-Rank \revision{Activation} from Implicit Regularization}

\label{sec:low_rank_imp_reg}


Building on this background we developed above, we now consider the role of the implicit regularization of gradient descent in promoting low-rank structure.  In particular, 
\revision{several works \cite{pmlr-v80-gunasekar18a,lyu2020gradient}} analyze classification models (i.e., when $\bm y$ is discrete) trained with gradient descent using 1) exponential family loss functions and 2) models whose predictions are \textit{positively homogeneous} with respect to the model parameters. 
Specifically, if we notate the network predictions as $\hat{\bm y} := f_{\bm \Theta}(\bm x)$, the authors of \cite{lyu2020gradient} consider loss functions that take the form $e^{g(\hat{ \bm y}, \bm y)}$ for functions $g$ that satisfy certain smoothness and monotonicity conditions. Common losses such as the cross-entropy loss and the exponential loss are special cases (see \cite{lyu2020gradient} for details).  
Moreover, the authors assume that the model is positively homogeneous in the model parameters\footnote{Recall that deep linear networks are positively homogeneous along with networks that use many common non-linearities, such as ReLU, leaky ReLU, max-pooling, etc.} -- that is there exists $p > 0$ such that for any model input $\bm x$, choice of model parameters $\bm \Theta$, and $\alpha \geq 0$ we have $f_{\alpha \bm \Theta}(\bm x) = \alpha^p f_{\bm \Theta}(\bm x)$. With the two conditions, the authors of \cite{lyu2020gradient} then show that if the model perfectly classifies the training data, then the particular solution found by gradient descent will be the maximum margin classifier (that is, gradient descent will find a classifier which perfectly classifies the training data with minimal squared norm on the model parameters). \revision{Moreover, in the history of the development of these ideas, the authors of \cite{pmlr-v80-gunasekar18a} showed a similar result for matrix factorization models, which is a special case of the results of \cite{lyu2020gradient}.} 

To better illustrate this general statement, consider the simplified setting of binary classification (i.e., $\bm y \in \{-1, 1\}$) with a deep linear network trained with the exponential loss,
\begin{equation}
\label{eq:exp_loss_linear}
\min_{\bm \Theta} F(\bm \Theta) = \sum_{i=1}^N \exp\{ -(\bm W_L \cdots \bm W_1 \bm x_i \cdot \bm y_i )\}.
\end{equation}
In this setting, we have that the model predictions are given by $f_{\bm \Theta}(\bm x) = \bm W_L \cdots \bm W_1 \bm x$ which is clearly positively homogeneous in the model parameters $\bm \Theta = \{ \bm W_L, \ldots, \bm W_1\}$.  Also, note that the level-sets for the exponential loss are unbounded, so there is no global minimizer of the loss and the magnitude of the parameters will grow unbounded during the optimization, so as a result it is necessary to consider the direction the parameters converge to by analyzing the normalized parameters $\bar {\bm \Theta} = \frac{\bm \Theta}{\| \bm \Theta \|_F}$, about which the following result can be shown.
%
%
%
\begin{theorem} (Informal - adapted from \cite{lyu2020gradient}) Consider the following training problem for a deep linear model trained with the exponential loss: 
\begin{equation}
\min_{\bm \Theta} F(\bm \Theta) = \sum_{i=1}^N \exp\{ -(\bm W_L \cdots \bm W_1 \bm x_i \cdot \bm y_i )\}.
\end{equation}
If the model is trained with gradient descent with sufficiently small step-size and at any point during training the model perfectly classifies the data (i.e., $\bm W_L \cdots \bm W_1 \bm x_i \cdot \bm y_i > 0, \ \forall \ i =1, \ldots, n$) then in direction the parameters will converge to a Karush–Kuhn–Tucker (KKT) point of the max-margin classification problem.  Specifically, any limit point of the parameter directions $\overline{\bm \Theta} = \frac{\bm \Theta}{\|\bm \Theta \|_F}$ will be a KKT point of the problem:
\begin{equation}
\label{eq:max_margin}
\min_{\bm \Theta} \frac{1}{2} \sum_{l=1}^L \|\bm \W_l \|_F^2 \ \ \text{s.t.} \ \ \bm W_L \cdots \bm W_1 \bm x_i \cdot \bm y_i \geq 1, \ \forall \ i =1, \ldots, n 
\end{equation}
That is, the parameter directions $\overline{\bm \Theta}$ found be gradient descent on \Cref{eq:exp_loss_linear} will be a KKT point of \Cref{eq:max_margin} to within a non-negative scaling - namely if $\bm \Theta_{KKT}$ is a KKT point of \Cref{eq:max_margin}, then $\exists \alpha > 0$  such that $\alpha \overline{\bm \Theta} = \bm \Theta_{KKT}$.
\label{thm:linexp}
\end{theorem} 

Here again, one can note that although there is no explicit regularization on the model parameters, the implicit regularization of the gradient descent dynamics finds a direction of the parameters which is a KKT point the problem in \Cref{eq:max_margin} which seeks to find a classifier with minimal $l_2$ norms on the model parameters (the maximum margin classifier)\footnote{The work of \cite{lyu2020gradient} also provides similar results for more general settings, such as alternative loss functions and multi-class classification with general positively-homogeneous models, and we refer the reader to \cite{lyu2020gradient} for the full technical details.}. Note that \Cref{eq:exp_loss_linear} and \Cref{eq:max_margin} are both non-convex so we are not guaranteed that gradient descent will find a global minimum in general, just a KKT point.  

Now from our discussion above in \Cref{sec:var_form}, recall that adding $l_2$ regularization to multi-linear model parameters results in low-rank solutions via the variational form of the Schatten-p (quasi)norms.  \revision{More specifically, the initial work of \cite{pmlr-v80-gunasekar18a} it is shown that for a matrix factorization model (i.e., $L=2$) then one recovers a nuclear norm regularized solution (see \cite{pmlr-v80-gunasekar18a} for full details). In addition, for a depth $L>2$} linear model, we can directly apply the result from \eqref{eq:variational-shatten-p} to the product of the first $L-1$ layers to see that the problem will be regularized by the Schatten-$\tfrac{2}{L-1}$ quasi-norm.  Namely, if the data can be linearly classified during training, then (in direction) the model parameters will converge to a KKT point of the problem
\begin{equation}
\min_{\bm W_L, \bm M} \frac{1}{2} \|\bm W_L \|_F^2 + \frac{2}{L-1} \left(\| \bm M \|_{\tfrac{2}{L-1}} \right)^{\tfrac{2}{L-1}}  \ \ \text{s.t.} \ \ \bm W_L \bm M \bm x_i \cdot \bm y_i \geq 1, \ \forall \ i =1, \ldots, N.
\end{equation}
Note that this can be interpreted as first finding a low-rank projection of the data $\bm M \bm x$, where $\bm M$ is regularized to be low-rank via the Schatten-$\tfrac{2}{L-1}$ quasinorm and then classifying that low-rank projection via the linear max-margin classifier $\bm W_L$.

\subsection{\revision{Low-Rank Activation from} Masked Training}

\label{sec:mask_train}

Another area where low-rank structure naturally arises in deep learning is when the data representation is \textit{masked} during training. More precisely, one stochastically replaces a portion of the data representation with default values such as zeros or noise, either in the space of the raw data or a latent representation of the data (e.g., in an intermediate layer of a deep network). This paradigm has been widely adopted in modern machine learning, as illustrated in the two examples. 

\begin{example}[\bf Dropout] One widespread use of masking in deep learning is the Dropout algorithm \cite{srivastava_dropout_2014}, where at each iteration of training a random subset of the \textit{neuron activations} is set to zero, while all neurons remain active during inference. This is motivated by the intuition that the induced stochasticity from masking portions of the latent representation would improve robustness and prevent overfitting. The technique has been successfully applied across various neural architectures, including convolutional neural networks (CNNs) for computer vision, recurrent neural networks (RNNs) for sequence modeling, and transformer models for natural language processing. 
Dropout has also inspired variants that apply different masking strategies, such as DropConnect, which randomly masks weights rather than activations, and Variational Dropout, which learns per-parameter dropout rates through variational inference. 
\end{example}

\begin{example}[\bf Masked self-supervised learning] Masked self-supervised training refers to the idea that one randomly replaces a subset of the \textit{input} with a default value and trains the model to reconstruct the unobserved input from the rest. It has become a fundamental technique in modern language and image modeling. Indeed, language models are trained by learning to predict unobserved text tokens using either the surrounding (\revision{Bidirectional Encoder Representations from Transformers}) 
or previous tokens (\revision{Generative Pre-trained Transformers}) 
in the sequence. The success of this paradigm motivates its extension to image modeling, where the models are trained to reconstruct unobserved image patches from observed ones (\revision{Masked Auto-Encoders}). 
Intuitively, the masking forces the network to learn efficient representations that allows inferring the correlation between the observed and unobserved data. 
\end{example}


\begin{figure}
\centering
\begin{subfigure}{0.24\textwidth}
    \centering
    \includegraphics[width=\linewidth, trim={0.3cm, 0.3cm, 0.3cm, 0.3cm}, clip]{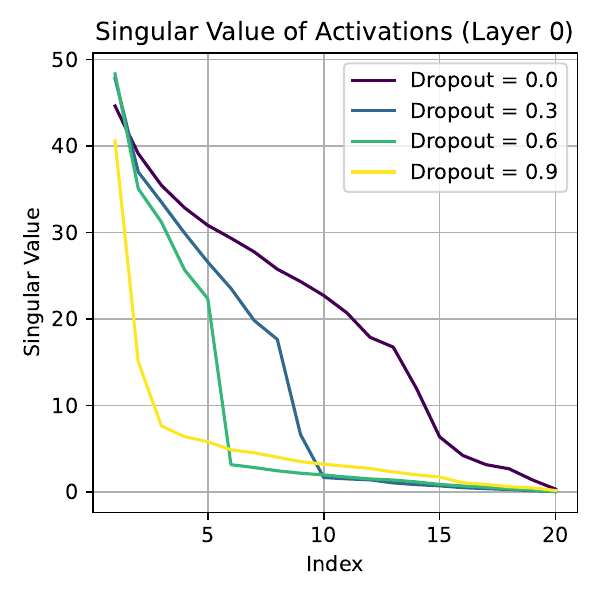}
    \caption{DLN (First Layer)}
    \label{fig:masked-dln-shallow}
\end{subfigure}
\hfill
\begin{subfigure}{0.24\textwidth}
    \centering
    \includegraphics[width=\linewidth, trim={0.3cm, 0.3cm, 0.3cm, 0.3cm}, clip]{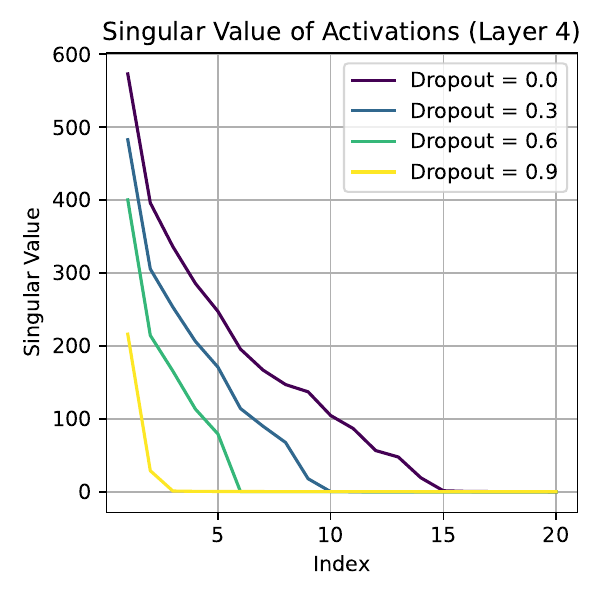}
    \caption{DLN (Last Layer)}
    \label{fig:masked-dln-deep}
\end{subfigure}
\begin{subfigure}{0.24\textwidth}
    \centering
    \includegraphics[width=\linewidth, trim={0.3cm, 0.3cm, 0.3cm, 0.3cm}, clip]{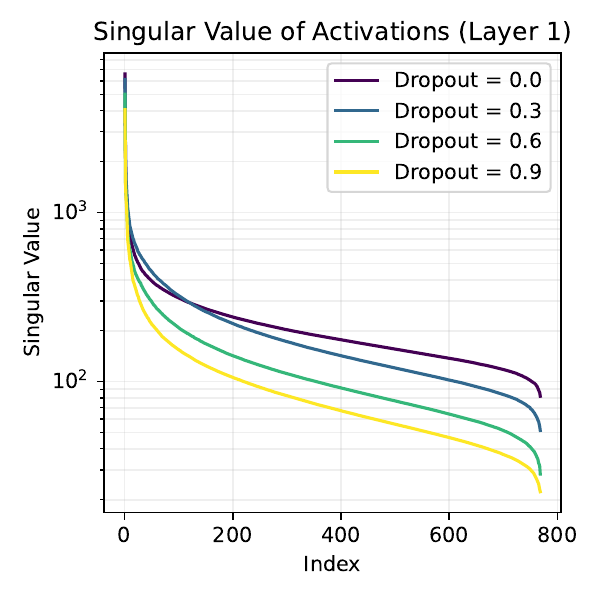}
    \caption{ResNet (First Layer)}
    \label{fig:masked-resnet-shallow}
\end{subfigure}
\begin{subfigure}{0.24\textwidth}
    \centering
    \includegraphics[width=\linewidth, trim={0.3cm, 0.3cm, 0.3cm, 0.3cm}, clip]{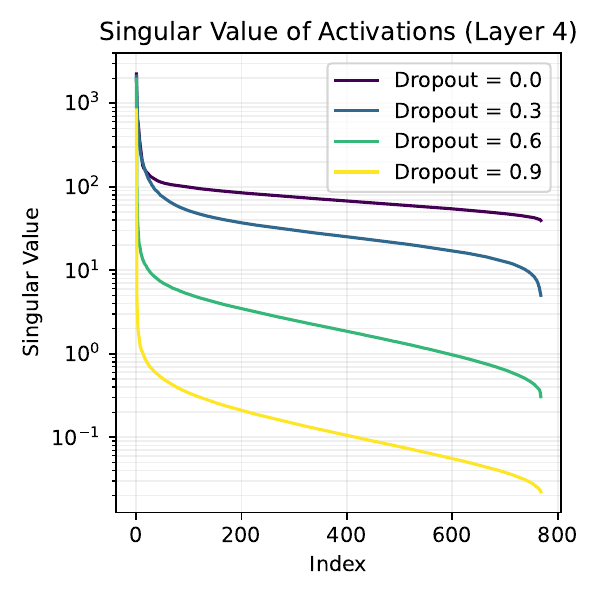}
    \caption{ResNet (Last Layer)}
    \label{fig:masked-resnet-deep}
\end{subfigure}
\caption{\textbf{Singular values of activations of the first and last layers in deep architectures from masked training with varying dropout rates $1-\mu$.} Deep Linear Network (DLN) is trained on synthetic data with MSE loss, while ResNet is trained on CIFAR-10 dataset of natural images with cross-entropy loss. As seen, the activations of deep and shallow layers of the trained network exhibit lower numerical rank when the dropout rate used in training is higher.}
\label{fig:masked-training-vary-dropout}
\end{figure}

To understand the effect of masking in training a network, initial works 
considered Dropout training in the simplified case of training a matrix factorization model (or equivalently a single hidden layer linear neural network) with the squared loss \cite{cavazza_dropout_2018}
In this context, the authors of \cite{cavazza_dropout_2018} observe that given a set of training data $(\X,\Y)$ then Dropout training can be viewed as minimizing the following objective:
%
\begin{equation}\label{eq:DropoutNN}
\min_{\W_1,\W_2} \E_{\z \stackrel{\text{i.i.d.}}{\sim} \text{Ber}(\mu)} \left \| \Y - \frac{1}{\mu} \W_2 \Diag(\z) \W_1 \X \right\|_F^2.
\end{equation}
Specifically, $\z$ is a vector of i.i.d. Bernoulli random variables with mean $\mu$, so Dropout training can be viewed as performing Stochastic Gradient Descent (SGD) on this objective by stochastically sampling a vector of $\z$ variables at each iteration of the optimization algorithm (i.e., randomly zeroing out the activations of the associated hidden units).  Then, simply expanding the quadratic term and evaluating this expectation leads to an equivalent deterministic regularized problem 
\begin{equation}\label{eq:DropoutNN-deterministic}
\min_{\W_1,\W_2} \|\Y - \W_2\W_1\X\|_F^2 + \frac{1-\mu}{\mu} \sum_{i=1}^d \|(\W_2)_{:,i}\|_2^2\|\X^\top (\W_1)_{i,:}\|_2^2.
\end{equation}
where one solves the original matrix factorization objective plus a regularization on the 
$d$ columns of the $\W_2$ matrix and $d$ rows of the $\W_1$ matrix.
Moreover, again connecting this with the variational form of the nuclear norm described above, the authors of \cite{cavazza_dropout_2018} and \cite{mianjy_dropout_2019} establish that the above problem is equivalent to a nuclear norm regularized problem.  In particular, the following problems are equivalent provided that $d$, the number of columns/rows in $(\bm W_2, \bm W_1)$ respectively, is sufficiently large:
%
\begin{equation}\label{eq:dropout-training-equivalent-nuclear-norm}
\min_{\W_1,\W_2} \E_{\z \stackrel{\text{i.i.d.}}{\sim} \text{Ber}(\mu)} \left \| \Y - \frac{1}{\mu} \W_2 \Diag(\z) \W_1 \X \right\|_F^2 \iff \min_{\Z} \|\Y - \Z\|_F^2 + \frac{1-\mu}{\mu} \| \Z\|_*^2,
\end{equation}
%
where the equivalence is in the sense that optimal solutions $(\W_1^*,\W_2^*)$ and $\Z^*$ will satisfy $\Z^* = \W_2^* \W_1^* \X$.
From this, one can observe that the use of Dropout training induces low-rank solutions due to the low-rank promoting properties of the nuclear norm.  

Moreover, this work was then later extended to provide a similar characterization for Dropout applied to deep linear networks \cite{mianjy_dropout_2019} and with Dropout applied to a final layer of a non-linear network with potentially different statistics (i.e., non-Bernoulli) on the Dropout variables \cite{pal_regularization_2020}.  In all cases, Dropout is shown to induce low-rank solutions, and the authors of \cite{arora_dropout_2021} further show that this also constrains statistical model complexity measures of the network, such as the Rademacher complexity. 

We conduct experiments to visualize the effect of dropout on the rank of the activations. We begin with the regime covered by the theoretical results reviewed above (i.e., last layer of deep linear network trained by MSE loss), and then explore cases not yet modeled by existing theory (i.e., shallower layers, non-linear networks such as multi-layer perceptrons (MLP) and residual network (ResNet), and cross entropy loss). Experiment setups are detailed in \Cref{supp:exp-details} in the supplementary materials. As seen in \Cref{fig:masked-training-vary-dropout}, the activations of both deep and shallow layers of the trained network exhibit lower numerical rank when the dropout rate used in training is higher. \Cref{fig:masked-training-vary-iterations} further shows that the activations of the layers gradually become low-rank as masked training proceeds.

More recently, other variants of masking data or latent representations of data have received considerable attention in applications of Masked Autoencoders (MAEs) and Masked Language Modelling (MLM). In this setting, empirical evidence has suggested that the features learned by MAEs and MLMs are approximately low-rank, suggesting that a potential general property of masking is to induce low-rank structures, although it remains an open direction for future work to theoretically establish this rigorously. For instance, the authors of \cite{zhang_how_2022} observe empirically that the features of MAEs at convergence have lower effective rank than at random initialization. They provide a theoretical connection between MAEs and contrastive learning, albeit how this relates to the low-rank features remains unclear. 
On the other hand, the work of \cite{meng_representation_2024} shows that on a pretrained MLM, inputs with masking empirically lead to features of lower effective rank than the same inputs without masking at all, with the difference of ranks being larger as the layer goes deeper. They prove that the features of the unmasked inputs must be rank-deficient at some layer, though it remains an open question to theoretically guarantee how small the rank can be. 

\section{Open Questions and Future Directions}\label{sec:open}


In this paper, we explored the role of low-rank structure in the training and adaptation of deep learning models. We reviewed recent advances in understanding low-rank dynamics during training and highlighted how implicit low-rank structures can reduce the number of required data samples for model learning. Additionally, we discussed methods such as LoRA that optimize large-scale models more efficiently by using low-rank approximations of the weight matrices. However, low-rank structures are also relevant to other important challenges in deep learning, such as improving model generalization and interpretability. In this section, we highlight some compelling open questions and propose promising future directions for research, with the aim of motivating more research in these crucial areas. 

One promising direction is to extend the theoretical study of low-rank structures from deep linear networks to deep nonlinear networks. For instance, in \Cref{sec:train_dyn}, we demonstrated that gradient updates in deep linear networks occur within a low-rank subspace. Although the Galore work in \Cref{sec:low-rank-training} provides some insights into the emergence of low-rank gradients in deep nonlinear networks, these findings depend on relatively restrictive assumptions, such as network reversibility, specific loss functions, and small batch sizes.  
Extending the analysis to more practical settings would provide better theoretical justification for employing low-rank training, \revision{yet the analysis in the nonlinear case is challenging for several reasons. For one, the exact low-rank structure in the gradient is lost, requiring greater care to ensure that the nonlinearity perturbs the gradient in such a way to preserve some notion of approximate low-rank structure. Additionally, certain algebraic properties that are taken for granted in the linear case, such as rotational symmetry and alignment between adjacent layers, are lost in the nonlinear setting, suggesting that it may be necessary to forgo these properties when making a more general argument. Finally, the manner in which singular values and subspaces ``pass'' through generally classes of typical nonlinearities used in deep learning is not well understood.}

Another promising direction lies in improving model interpretability through low-rank structures, which could offer simpler and more human-understandable explanations of the learned features or transformations in complex architectures, such as diffusion models and large language models. For example,  the interplay between low-rank structures and phenomena like neural collapse 
\cite{zhu2021geometric} in supervised learning offers fertile ground for advancing our understanding of these models' underlying mechanics. 
Additionally, leveraging low-rank approximations to improve generative AI efficiency has gained traction, particularly in the context of LLMs, where self-attention maps often exhibit low-rank characteristics. This observation has inspired techniques that approximate self-attention matrices to reduce computational complexity while maintaining performance. 

\newpage 

\section*{Acknowledgment}
QQ acknowledges NSF CAREER CCF-2143904, NSF IIS-2312842, NSF IIS-2402950, NSF CCSS-2532643, and Google Research Scholar Award. LB, CY, and SK acknowledge NSF CAREER CCF-1845076, NSF CCF-2331590. LB acknowledges the U-M Crosby Award, Intel Early Career award, and the IAS Charles Simonyi Endowment. PW acknowledges the
University of Macau SRG2025-00043-FST and UMDF-TISF-I/2026/013/FST, and the Macau Science and Technology Development Fund (FDCT) 0091/2025/ITP2. ZW acknowledges NSF CAREER 2145346, NSF DMS-02133861, NSF CCSS-2113904, and the NSF AI Institute for Foundations of Machine Learning (IFML). TD and BDH acknowledge NSF 2212457.

The authors would like to thank the anonymous reviewers for their excellent suggestions and for clarifying the organization and purpose of the paper, as well as Zekai Zhang for his insights into the relationship between ReLoRA and GaLore. We also disclose that we used GPT5 on \Cref{sec:lorarank} when revising our paper in response to the reviewers in order to compress the section, followed by several readings, verifications, edits, and refinements by three of our authors.

\printbibliography

\appendix
\section{Experiment Details}\label{supp:exp-details}

\begin{table}[H]
\centering
\caption{Summary of theoretical and empirical results in the paper. While the theoretical results are for deep linear networks, one can observe emergent low-rank structures empirically when the networks have non-linear activations. 
}
\label{tab:summary-emp-results}
\begin{tabular}{@{}llll@{}}
\toprule
Section & Non-linear & Theoretical & Empirical Result \\
 & Activation & Result &  \\ \midrule
\S \ref{sec:train_dyn} Low-rank weight matrix & No & \Cref{thm:lop} & \Cref{fig:svals} (left), \Cref{fig:svd} \\  & Yes &  & \Cref{fig:svals} (middle left, middle right, right) \\ \midrule
\S \ref{sec:lora} Low-rank adaptation & Yes & & \Cref{fig:svd_lora,fig:improving_lora,fig:low_rank_subspace} \\
\S \ref{sec:low-rank-training} Low-rank training & Yes & & \Cref{fig:vitfc2,fig:vitvalue} \\ \midrule
\S \ref{sec:mask_train} Low-rank activation & No & \Cref{cor:variational-shatten-p} & \Cref{fig:masked-training-vary-iterations} (a), \Cref{fig:masked-training-vary-dropout} (a, b) \\
& & \Cref{thm:linexp} & \\
 & Yes & & \Cref{fig:masked-training-vary-iterations} (b, c), \Cref{fig:masked-training-vary-dropout} (c, d) \\ \midrule
 &  & 
\end{tabular}
\end{table}

\noindent\textbf{Low-Rank structure at Every Iteration of Gradient Dynamics ( \Cref{sec:LRtrain})}
For \Cref{fig:svals}, the linear network and MLP were trained on MNIST and the VGG-16 and ViT were trained on CIFAR-10.  We observe the change in singular values of the penultimate weight matrix for DLNs, MLPs, VGG, and ViT-B. We train the DLN starting from orthogonal initialization using SGD with learning rate $0.01$ and batch size $128$. All of the nonlinear networks were initialized using random uniform initialization. For MLP, we train using the same parameters as the DLN. For VGG, we use batch size $128$ with learning rate $0.05$, weight decay parameter $5 \times 10^{-4}$, momentum $0.9$ and step size scheduler with cosine annealing. Lastly, for ViT-B, we train using ADAM with a batch size of $512$ and learning rate $10^{-5}$ with cosine annealing. The ViT-B architecture is the same as the one presented by Dosovitskiy et al.~\cite{DBLP:conf/iclr/DosovitskiyB0WZ21}. \Cref{fig:vitfc2,fig:vitvalue}  follow the same training procedure for the ViT above.

\begin{table}[h]
\centering
\caption{Summary of setups of experiments in \Cref{sec:LRobj}.}
\label{tab:exp-setup-masked}
\begin{tabular}{@{}lcccc@{}}
\toprule
Experiment & Architecture & Dataset & Loss Function & Batch Size \\ \midrule
1 & Deep Linear Network & Synthetic & MSE & 1024 \\
2 & Multi-Layer Perceptron & Synthetic & MSE & 1024 \\
3 & ResNet & CIFAR-10 & Cross-Entropy & 768 \\ \bottomrule
\end{tabular}
\end{table}

For \Cref{fig:low_rank_subspace}, we fine-tune BERT \cite{devlin-etal-2019-bert} with deep overparameterized adaptation on the STS-B dataset \cite{stsb}. The pretrained BERT and T5 models (and tokenizer) are retrieved from the \texttt{transformers} library \cite{wolf2019huggingface} as \texttt{google-bert/bert-base-cased} and \texttt{google-t5/t5-base} respectively. We choose the best learning rate for each method from $\eta \in \{10^{-5}, 10^{-4}, 10^{-3}, 10^{-2}\}$ on STS-B with 1024 samples, and find that $\eta = 10^{-4}$ and $\alpha=8$ works best for vanilla LoRA, while we need two parameters for Deep LoRA's discrepant learning rate (see \cite{yaras2024compressible} for details): $\eta = 10^{-2}$ with $\gamma = 10^{-2}$ works best for Deep LoRA (although $\gamma$ can be chosen relatively freely). Vanilla LoRA is initialized in the same fashion as the original paper (i.e., $\bm W_k^{(2)}$ is initialized to all zeros, $\bm W_k^{(1)}$ is initialized to be Gaussian with standard deviation 1), whereas Deep LoRA is compressed from a full-width 3-layer factorization with orthogonal initialization of scale $\epsilon_l = 10^{-3}$ using the subspace corresponding to the smallest singular values of the gradient at initialization, inspired from \Cref{thm:lop}. We use a train batch size of 16, and train all models until convergence in train loss, and use the final model checkpoint for evaluation. These details and more experiments can be found in \cite{yaras2024compressible}.

\noindent\textbf{Low-Rank structure at the Global Minimum of Objective Functions (\Cref{sec:LRobj})} We design the experiments to investigate the impact of dropout on the singular values (or effective rank) of activations in neural networks, where we progressively extend theoretical predictions to practical scenarios. \Cref{tab:exp-setup-masked} summarizes the architecture, dataset, and loss function of each experiment. \textit{Synthetic Data Generation:} For Experiments 1 and 2, we generate synthetic datasets. Inputs \( x \in \mathbb{R}^{20} \) are drawn from a standard Gaussian distribution. The corresponding outputs \( y \in \mathbb{R}^{20} \) are generated via a linear transformation with a fixed matrix of rank 15. This design ensures outputs lie in a $15$-dimensional subspace of $\mathbb{R}^{20}$. \textit{Network Architecture and Loss Function:}
In Experiment 1, we utilize a deep linear network comprising three linear layers with dimensions 20-20-20, separated by dropout layers, and optimized using mean squared error (MSE) loss. Experiment 2 introduces nonlinearities into the network by inserting ReLU activations after each intermediate linear layer, except for the output layer, thus forming a multilayer perceptron (MLP). Experiment 3 adopts a practical setting using a ResNet architecture with residual blocks arranged in a 2-2-2-2 configuration, channel dimensions starting at 64 and doubling progressively (64, 128, 256, 512). This network is trained on the CIFAR-10 dataset, employing cross-entropy loss to reflect a realistic classification scenario. In all three experiments an SGD optimizer is used with a learning rate of $10^{-2}$ and no momentum whatsoever.

\section{Complexity for Forward and Backward Pass}
\label{app:gencomplexity}
Almost all of the computational cost (and a majority of the memory cost) in training deep networks involves dense linear layers of the form $\bm Y = \bm W \bm X$, where $\bm X \in \mathbb{R}^{d \times n}$ are the input activations (i.e., the batch of hidden feature vectors after passing through a nonlinearity from a previous layer), $\bm Y \in \mathbb{R}^{d \times n}$ are output activations, and $\bm W \in \mathbb{R}^{d \times d}$ is a weight matrix. 
During forward propagation, we receive input activations $\bm X$ from the previous layer, compute $\bm Y$ with $\Theta(nd^2)$ operations, which is passed to the subsequent layer, and store $\bm X$ for backpropagation. 
During backpropagation, we receive the output activation gradient $\nabla_{\bm Y} \mathcal{L}$ with respect to some loss $\mathcal{L}$ from the subsequent layer, which we use to compute the weight matrix gradient $\nabla_{\bm W} \mathcal{L} = \nabla_{\bm Y} \mathcal{L} \cdot \bm X^\top$ with $\Theta(nd^2)$ operations, and also the input activation gradient $\nabla_{\bm X} \mathcal{L} = \bm W^\top \cdot \nabla_{\bm Y}\mathcal{L}$ with $\Theta(nd^2)$ operations, with the latter passed to the preceding layer. 
Overall, for each layer we have $3 \times \Theta(nd^2)$ operations. From a memory standpoint, for each layer we need to store $\Theta(d^2)$ for the weight matrix (and its optimizer state which is $2\times$ the size of weight matrix for Adam), as well as $\Theta(nd)$ for input activations. 

There are some ``tricks'' that can be (and are) employed in practice to save memory, which slightly change the above setting. One example is ``gradient checkpointing'', where we only store input activations at certain layers (called checkpoints), which we can use to recompute a given layer's input activations for backpropagation. We are effectively trading memory for compute -- if we only store every other layer's input activations, we halve the memory requirements for these, while introducing another matrix multiplication for every other layer.

\section{Discussion on Low-Rank Gradient Dynamics for Deep Matrix Factorization}
\subsection{Example for $L=2$}
In this section we provide the specific details of the proof of \Cref{thm:lop} for $L=2$. The general proof is found in \cite{yaras2024compressible}. The key is that we identify the $(d-2r)$-dimensional subspace where gradient updates {\em do not} occur. Note we will now repeatedly rely on the properties of our initialization, that weight matrices are orthogonal and scaled by $\epsilon_l$. The gradient with respect to $\bm W_1^{(0)}$ and $\bm W_2^{(0)}$ is thus given by
\begin{align*}
    \nabla_{\bm W_1} \mathcal{L}(\bm \Theta^{(0)}) &= \bm W_2^{(0)\top} (\bm W_2^{(0)} \bm W_1^{(0)} - \bm \Phi) = \epsilon_2^2 \bm W_1^{(0)} - \bm W_2^{(0)\top}\bm \Phi, 
    \; \mbox{and}\\
    \nabla_{\bm W_2} \mathcal{L}(\bm \Theta^{(0)}) &= (\bm W_2^{(0)} \bm W_1^{(0)} - \bm \Phi) \bm W_1^{(0)\top} = \epsilon_1^2 \bm W_2^{(0)} - \bm \Phi \bm W_1^{(0)\top}.
\end{align*}
From here, we can construct $(d-2r)$-dimensional left and right subspaces $\mathcal{U}_1$ and $\mathcal{V}_1$ of $\bm W_1^{(0)}$ that lie in the left and right nullspaces of $\nabla_{\bm W_1} \mathcal{L}(\bm \Theta^{(0)})$ respectively. 
Specifically, let $\mathcal{V}_1 = \mathcal{N}(\bm \Phi) \cap \mathcal{N}(\bm \Phi^\top \bm W_2^{(0)} \bm W_1^{(0)})$, where $\mathcal{N}$ denotes nullspace, and let $\mathcal{U}_1 = \sspan(\{\bm W_1^{(0)}\bm v, \forall \bm v \in \mathcal{V}_1\})$, which we will denote with shorthand $\mathcal{U}_1 = \{\bm W_1^{(0)} \mathcal{V}_1\}$. As a result, $\mbox{dim}(\mathcal{U}_1) = \mbox{dim}(\mathcal{V}_1) \geq d - 2r$. 

This proper choice of subspaces directly results in the decomposition of \Cref{thm:lop}. The way we defined it, there are are at least $d-2r$ pairs of vectors $(\bm u, \bm v) \in \mathcal{U}_1 \times \mathcal{V}_1$ with $\bm W_1^{(0)} \bm v = \epsilon_1 \bm u$, while $\nabla_{\bm W_1} \mathcal{L}(\bm \Theta^{(0)}) \bm v = \epsilon_2^2 \epsilon_1 \bm u$ and 
$$\nabla_{\bm W_1} \mathcal{L}(\bm \Theta^{(0)})^\top \bm u = \epsilon_2^2 \bm W_1^{(0)\top} \bm u - \Phi^\top \bm W_2^{(0)} \bm u = \epsilon_2^2 \epsilon_1 \bm v  - \Phi^\top \bm W_2^{(0)} (\epsilon_1 \bm W_1^{(0)}\bm v) = \epsilon_2^2 \epsilon_1 \bm v$$ 
by using the second intersected null space in the definition of $\mathcal{V}_1$. 
Therefore, \cref{eq:deep_mat_grad} gives us:
\begin{align*}
    \bm W_1^{(1)} \bm v &= \left((1-\eta\lambda) \bm W_1^{(0)} - \eta \nabla_{\bm W_1} \mathcal{L}(\bm \Theta^{(0)}) \right) \bm v = \epsilon_1 (1-\eta\lambda -\eta \epsilon_2^2) \bm u, \\
    \bm W_1^{(1)\top} \bm u &= \left((1-\eta\lambda) \bm W_1^{(0)\top} - \eta \nabla_{\bm W_1} \mathcal{L}(\bm \Theta^{(0)})^\top \right) \bm u = \epsilon_1 (1-\eta \lambda-\eta \epsilon_2^2) \bm v ,
\end{align*}
meaning that $(\bm u, \bm v) \in \mathcal{U}_1 \times \mathcal{V}_1$ remain singular vectors after the gradient step, but their singular value shrinks by a factor of $1-\eta\lambda-\epsilon^2$ (independent of $\bm \Phi$). Similar invariant subspaces emerge in the second layer -- in fact, the invariant subspaces align across the two layers. Taking $\mathcal{V}_2 = \mathcal{U}_1$, so that 
\begin{align*}
    \{\bm W_1^{(0)\top} \mathcal{V}_2\} = \{\bm W_1^{(0)\top} \mathcal{U}_1\} = \mathcal{V}_1 \subset \mathcal{N}(\bm \Phi)
\end{align*}
by construction of $\mathcal{V}_1$, we have that any $\bm v \in \mathcal{V}_2$ satisfies $\nabla_{\bm W_2} \mathcal{L}(\bm \Theta^{(0)}) \bm v = \epsilon_1^2 \bm W_2^{(0)} \bm v = \epsilon_1^2 \epsilon_2 \bm u$ defining $\mathcal{U}_2$ in the analogous way to the first step. While this argument is for $L=2$, the result can be generalized to arbitrary $L$, as can be seen in the supplementary material of \cite{yaras2024compressible}. We emphasize once again that we define the subspaces where training does not occur and subsequently determine the invariant training subspaces as their orthogonal complements.

\subsection{Discussion of Assumptions in \Cref{thm:lop}}

We highlight the key ingredients for the emergence of low-rank structure in the optimization of deep matrix factorizations. First, we require that the target data $\bm \Phi$ is low rank -- when $\bm \Phi = \bm Y \bm X^\top (\bm X \bm X^\top)^{-1}$, this is satisfied due to the columns of $\bm Y$ often lying in a much lower dimensional subspace compared to the input dimensionality/network width $d$, and one can show a similar result to \Cref{thm:lop} when $\bm Y$ has $k \ll d$ rows -- see the next \Cref{app:lop_classification} for more details. The second assumption we make is that the weight matrices are initialized to be scaled orthogonal matrices -- this is approximately satisfied when (1) the entries of $\bm W_l^{(0)}$ are drawn \emph{iid} from a Gaussian distribution, which is utilized in practice, and (2) $d$ is large, corresponding to the wide and overparameterized nature of modern deep networks.
To ensure that the final weights are low-rank, we require some form of regularization, whether that be implicit via small-scale initialization or explicit via weight-decay. Finally, we have assumed that the objective is of the form \Cref{eq:deep_mat_factor}, where we employ the squared error loss with respect to some target. In \Cref{sec:lorarank} this was generalized to the low-rank adaptation setting that with Deep LoRA, but only with algorithmic developments, and not in theory.

\subsection{Low-rank Gradient Dynamics with a Wide Target}\label{app:lop_classification}
In \Cref{sec:train_dyn}, we consider deep matrix factorization where the target matrix $\bm \Phi$ has $d$ rows. In a more standard regression or classification setting for $\bm \Phi = \bm Y \bm X^\top (\bm X \bm X^\top)^{-1}$, we really have $k$ rows, where $k$ is the label dimension (e.g., corresponding to $k$ classes). In most cases, $k \ll d$, i.e., we have many fewer output coordinates than the network width or input dimensionality. To address this case, suppose $\bm \Phi \in \mathbb{R}^{k \times d}$, where $r=k \ll d$ is the rank of $\bm \Phi$. We modify $\bm W_L$ in \Cref{eq:deep_mat_factor} to be $\bm W_L \in \mathbb{R}^{k \times d}$ to match dimensions. Then, a very similar result to \Cref{thm:lop} can be shown, with the only differences being that (1) the decomposition holds for all layers $l \in [L-1]$, i.e., excluding layer $L$, and (2) $\rho_l^{(t)}=\epsilon_l(1-\eta \lambda)^t$ \emph{exactly}. 
This result and its proof are found in \cite{yaras2023law}.

\begin{figure}[t]
\includegraphics[width=\linewidth]{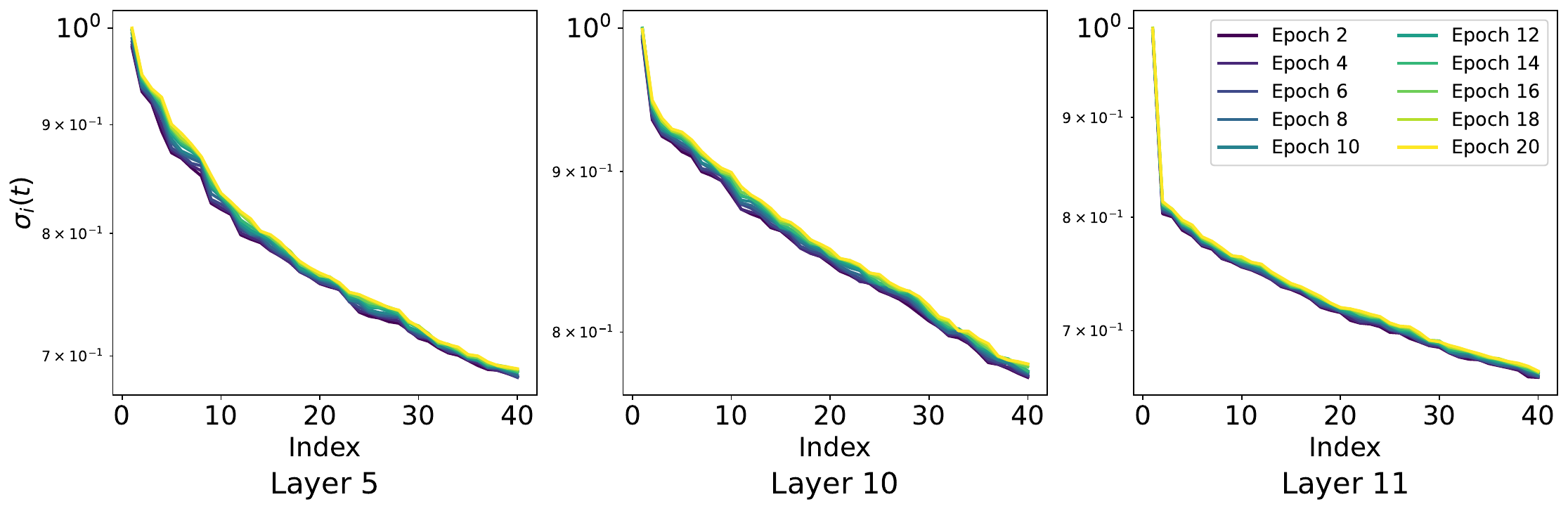}
     \caption{\textbf{Singular values across layers of the ViT of the value weight matrix.} Similar to the observations in \Cref{fig:vitfc2}, the value matrices corresponding to the Transformer of the shallower layers do not have a low-rank structure, and the low-rank structure emerges in the deeper layers.}
    \label{fig:vitvalue}
\end{figure}
\section{Low-Rank Training}\label{app:low-rank-training}
This section of the appendix presents additional results that complement the low-rank training analysis in \Cref{sec:low-rank-training}. Specifically, we include the following:
\begin{itemize}[leftmargin=*]
    \item \emph{An illustration of rank differences in value matrices of ViT.} The difference of low-rank structures across the layers of value matrices in ViT models are shown in \Cref{fig:vitvalue}. 
    \item \emph{Algorithm pipeline of GaLore.} Additionally,
we provide more details of GaLore method based upon Adam optimizer in \Cref{alg:low_rank_adam}. 
    \item \emph{Equivalence of GaLore and ReLoRA with full rank gradient initialization.} Finally, here we include proof of the equivalence of ReLoRA \cite{lialin2024relora} and GaLore \cite{zhao2024galore} detailed in the following.
\end{itemize}

\begin{theorem}[GD-GaLore and GD-ReLoRA are equivalent with full rank gradient initialization]

Initialize GD-GaLore from \cite{zhao2024galore} and GD-ReLoRA from \cite{lialin2024relora} with the same initial point. In particular, initialize the $\bm B$ matrix in ReLoRA with the SVD of full gradients (rather than random Gaussian) and fix it. Then, the weight updates for each algorithm are identical at every iteration. 
\end{theorem}

\begin{proof}
We prove this by induction. We use the notation from the GaLore algorithm and denote quantities related to GD-ReLoRA with a prime (e.g., $\bm W^{(t)'}$).

\subsubsection*{Induction Hypothesis}
Assume that GD-GaLore and GD-ReLoRA start with the same initialization and produce the same result at iteration $t-1$:
\begin{align*}
    \begin{cases}
    \bm W^{(t-1)} = \bm W^{(t-1)'}  \\
    \bm P^{(t-1)} = \bm B^{(t-1)}.
    \end{cases}
\end{align*}
where \begin{align*}
    \bm W^{(k)'} &:= \bm W^{(0)}+\bm B^{(0)} \bm A^{(0)}+\bm B^{(T)}\bm A^{(T)}+\dots+\bm B^{([(k-1)/T]T)} \bm A^{([(k-1)/T]T)}+\bm B^{(k)}\bm A^{(k)}\\ 
    &:= \bm W^{([(k-1))/T]T)} + \bm B^{(k)}\bm A^{(k)}
\end{align*} is the $k$-th iterate of GD-ReLoRA.

\subsubsection*{Case 1: \( t \mod T \neq 0 \)}

For GD-GaLore, using the update rule:
\begin{align*}
    \bm G^{(t)} &= \nabla_{\bm W} \phi(\bm W^{(t-1)}), \bm P^{(t)} = \bm P^{(t-1)}, \\
    \bm R^{(t)} &= \bm P^{(t)\top} \bm G^{(t)}, \tilde{\bm G}_t = \bm P^{(t)} \bm G^{(t)}, \\
    \bm W^{(t)} &= \bm W^{(t-1)} - \eta \tilde{\bm G}_t = \bm W^{(t-1)} - \eta \bm P^{(t)}\bm P^{(t)\top} \nabla_W \phi(\bm W^{(t-1)}).
\end{align*}

\noindent For GD-ReLoRA, applying the update rule:
\begin{align*}
    \bm A^{(t)} &= \bm A^{(t-1)} - \nabla_{\bm A} \phi(\bm W^{([(t-2)/T]T)}+\bm B^{(t-1)}\bm A^{(t-1)}), \bm B^{(t)} = \bm B^{(t-1)}, \\
    \bm W^{(t)'} &= \bm W^{(t-1)'} + \bm B^{(t)} \bm A^{(t)} - \bm B^{(t-1)} \bm A^{(t-1)}, \\
    &= \bm W^{(t-1)'} - \eta \bm B^{(t)} \bm B^{(t)\top} \nabla_{\bm W} \phi(\bm W^{(t-1)'}).
\end{align*} where we leverage the chain rule $\nabla_{\bm A} \phi(\bm W^{([(t-2)/T]T)}+\bm B^{(t-1)}\bm A^{(t-1)}) = \bm B^{(t-1)\top}\nabla_{\bm W} \phi(\bm W^{(t-1)'})$.

\noindent By the induction hypothesis:
\begin{align*}
    \bm W^{(t)'} &= \bm W^{(t-1)'} - \eta \bm B^{(t)} \bm B^{(t)\top} \nabla_{\bm W} \phi(\bm W^{(t-1)'}) \\
    &= \bm W^{(t-1)} - \eta \bm P^{(t)}\bm P^{(t)\top} \nabla_{\bm W} \phi(\bm W^{(t-1)}) = \bm W^{(t)}.
\end{align*}

Thus, the claim holds for \( t \mod T \neq 0 \).

\subsubsection*{Case 2: \( t \mod T = 0 \)}

Now we consider the case where we need to reinitialize \( \bm P^{(t)} \) and \( \bm B^{(t)}, \bm A^{(t)} \).

\noindent For GD-GaLore, we use the SVD of the full gradient as projection:
\begin{align*}
    \bm G^{(t)} &= \nabla_{\bm W} \phi(\bm W^{(t-1)}), [\bm U,\bm \Sigma,\bm V] = \text{SVD}(\bm G^{(t)}), \\
    \bm P^{(t)} &= \bm U[:, :r], \bm W^{(t)} = \bm W^{(t-1)} - \eta \bm P^{(t)} \bm P^{(t)\top} \bm G^{(t)}.
\end{align*}

\noindent For GD-ReLoRA, we reinitialize \( \bm B^{(t)} \) using the SVD of the gradient and set \( \bm A^{(t)} \) to zero:
\begin{align*}
    \bm G^{(t)'} &= \nabla_{\bm W} \phi(\bm W^{(t-1)'}), [\bm U', \bm \Sigma', \bm V'] = \text{SVD}(\bm G^{(t)'}), \\
    \bm B^{(t)} &= \bm U'[:, :r], \bm A^{(t)\text{init}} = 0, \bm W^{(t)'\text{init}} = \bm W^{(t-1)'} + \bm B^{(t)} \bm A^{(t)\text{init}}, \\
    \bm A^{(t)} &= \bm A^{(t)\text{init}} - \eta \nabla_{\bm A} \phi(\bm W^{(t);\text{init}}), \\
    \bm W^{(t)'} &= \bm W^{(t)'\text{init}} + \bm B^{(t)} \bm A^{(t)} - \bm B^{(t)} \bm A^{(t)\text{init}}.\\
    & = \bm W^{(t-1)'} - \eta \bm B^{(t)} \bm B^{(t)\top} \nabla_{\bm W} \phi(\bm W^{(t-1)'}).
\end{align*}

Applying the induction hypothesis:
\begin{align*}
    \bm W^{(t)'} &= \bm W^{(t-1)'} - \eta \bm B^{(t)} \bm B^{(t)\top} \nabla_{\bm W} \phi(\bm W^{(t-1)'})\\
    & = \bm W^{(t-1)} - \eta \bm P^{(t)} \bm P^{(t)\top} \nabla_{\bm W} \phi(\bm W^{(t-1)}) = \bm W^{(t)}.
\end{align*}

Since Case 2 also serves as the base case, the proof is complete.
\end{proof}

\subsubsection*{Remarks}
\begin{itemize}
    \item This theorem shows that, due to the chain rule, we do not need to compute the full gradient first; instead, we can update the low-rank approximation directly.
    \item However, we still need to modify GD-ReLoRA by reinitializing \( B_t \) with the SVD of the full gradient every \( T \) iterations. This is significantly cheaper than computing the full gradient at every step.
    \item The proof extends naturally to Adam-GaLore and Adam-ReLoRA, leading to the same conclusion.
\end{itemize}

\begin{algorithm}[tb]
   \caption{Adam with GaLore (reprinted from \cite{zhao2024galore})}
   \label{alg:low_rank_adam}
 \begin{algorithmic}
   \STATE {\bfseries Input:} A layer weight matrix $\bm W \in \mathbb{R}^{m \times n}$ with $m \leq n$. Step size $\eta$, scale factor $\alpha$, decay rates $\beta_1, \beta_2$, rank $r$, subspace change frequency $T$.
   \STATE Initialize first-order moment $\bm M_0 \in \mathbb{R}^{n \times r} \gets 0$
   \STATE Initialize second-order moment $\bm V_0 \in \mathbb{R}^{n \times r} \gets 0$
   \STATE Initialize step $t \gets 0$
   \REPEAT
   \STATE $\bm G_t \in \mathbb{R}^{m \times n} \gets - \nabla_{\bm W} \phi_t(\bm W_t)$  \hfill \COMMENT{Should be $\nabla_{\bm W} \phi_t(\bm W_{t-1})$?}
   \IF{$t \bmod T = 0$}
   \STATE $\bm U, \bm S, \bm V \gets \text{SVD}(\bm G_t)$
   \STATE $\bm P_t \gets \bm U[:, :r]$ \hfill \COMMENT{Initialize left projector as $m \leq n$}
   \ELSE
   \STATE $\bm P_t \gets \bm P_{t-1}$ \hfill \COMMENT{Reuse the previous projector}
   \ENDIF
   \STATE $\bm R_t \gets \bm P_{t}^{\top} \bm G_t$ \hfill \COMMENT{Project gradient into compact space}
   \\\hrulefill
   \STATE {\bfseries $\text{update}(\bm R_t)$ by Adam}
   \hspace{\algorithmicindent} \STATE $\bm M_t \gets \beta_1 \cdot \bm M_{t-1} + (1 - \beta_1) \cdot \bm R_t$ 
   \hspace{\algorithmicindent} \STATE $\bm V_t \gets \beta_2 \cdot \bm V_{t-1} + (1 - \beta_2) \cdot \bm R_t^2$ 
   \hspace{\algorithmicindent} \STATE $\bm M_t \gets \bm M_t / (1 - \beta_1^t)$
   \hspace{\algorithmicindent} \STATE $\bm V_t \gets \bm V_t / (1 - \beta_2^t)$ 
   \hspace{\algorithmicindent} \STATE $\bm N_t \gets \bm M_t / (\sqrt{\bm V_t} + \epsilon)$
   \\\hrulefill
   \STATE $\tilde{\bm G}_t \gets \alpha \cdot P \bm N_t$ \hfill \COMMENT{Project back to original space}
   \STATE $\bm W_t \gets \bm W_{t-1} + \eta \cdot \tilde{\bm{G}}_t$
   \STATE $t \gets t + 1$
   \UNTIL{convergence criteria met}
   \RETURN $\bm W_t$
 \end{algorithmic}
\end{algorithm}

\section{Additional Figures}\label{app:additional-figures}

Here, we provide 2D versions of some figures from the main body. Specifically, \Cref{fig:svals_2d,fig:svd_2d,fig:svd_lora_2d,fig:low_rank_subspace_2d} correspond to \Cref{fig:svals,fig:svd,fig:svd_lora,fig:low_rank_subspace}, respectively.

\begin{figure}[h]
\includegraphics[width=\linewidth]{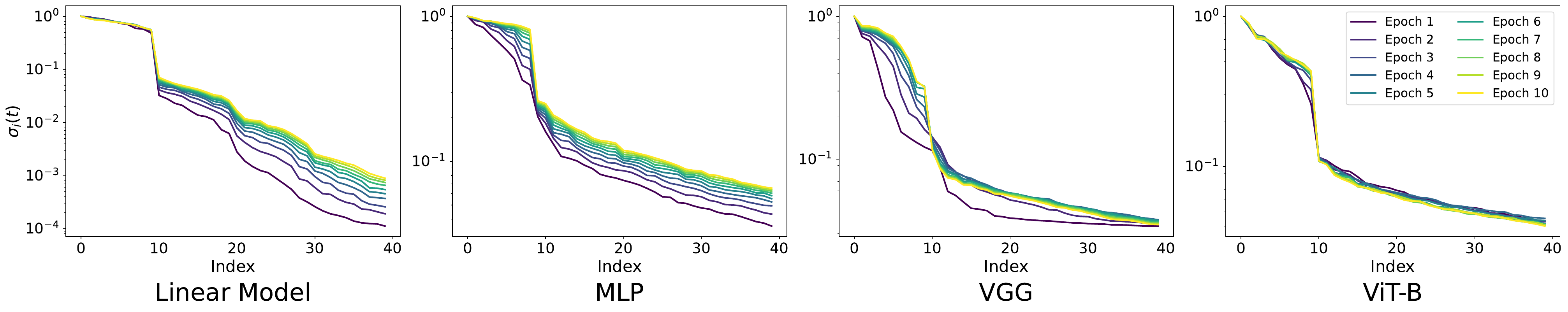}
     \caption{\textbf{Prevalence of low-rank weight updates in various deep networks.} Each plot visualizes the
     singular values of the weight updates from initialization for the penultimate layer weight matrix for different types of network architectures: deep linear network (DLN), multi-layer perception (MLP), VGG, and ViT-B. 
     We show them with two views: one in 3D emphasizes the evolution of the singular values over iteration, and the other using color shows the singular values on log-axis giving a clearer picture of the low-rank structure in each setting.
     The linear network is trained on MNIST with mean square error loss. The MLP is trained on MNIST with cross-entropy loss.
     The VGG and ViT-B networks are both trained on CIFAR-10 with cross-entropy loss. 
     The result shows a prevalent phenomenon across linear and nonlinear networks -- gradient descent only updates a small portion of the singular values, while the others remain small and almost unchanged. }
    \label{fig:svals_2d}
\end{figure}

\begin{figure}[h]
    \centering
    \includegraphics[width=\linewidth]{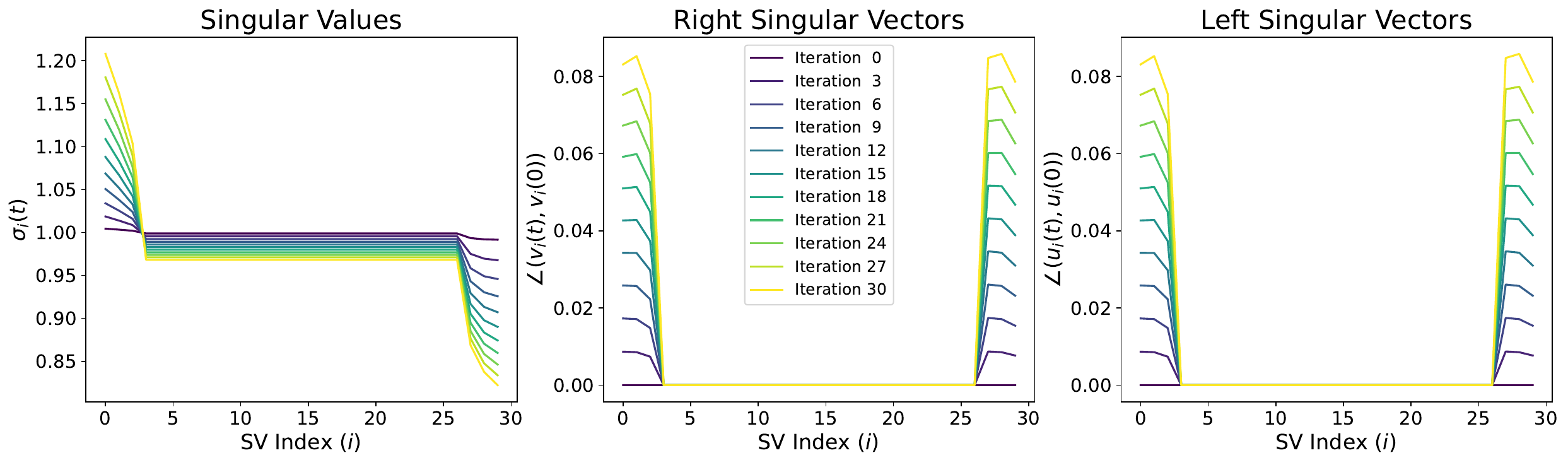}
    \caption{\textbf{Evolution of SVD of weight matrices for deep matrix factorization.}  We visualize the SVD dynamics of the first layer weight matrix of an $L=3$ layer deep matrix factorization for a random matrix with $d = 30$, $r=3$, $\epsilon_l = 1$ throughout GD without weight decay ($\lambda = 0$). \textit{Left}: Magnitude of the $i$-th singular value $\sigma_i(t)$ at iteration $t$. \textit{Middle}: Angle $\angle(\bm v_i(t), \bm v_i(0))$ between the $i$-th right singular vector at iteration $t$ and initialization. \textit{Right}: Angle $\angle(\bm u_i(t), \bm u_i(0))$ between the $i$-th left singular vector at iteration $t$ and initialization.}
    \label{fig:svd_2d}
\end{figure}

\begin{figure}[h]
    \centering
    \includegraphics[width=\linewidth]{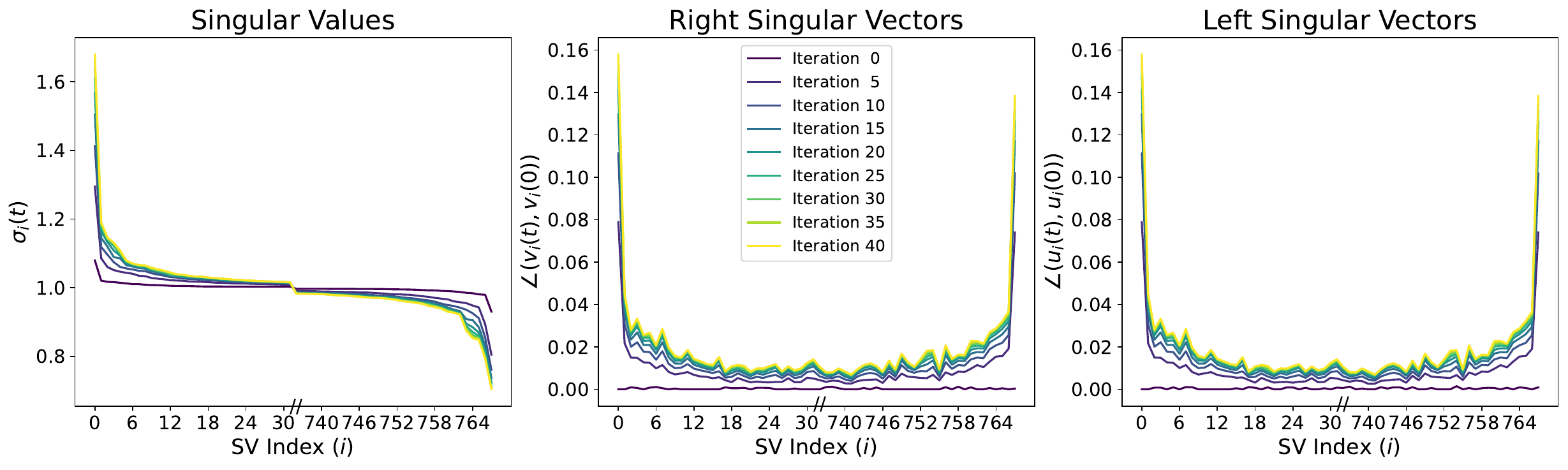}
    \caption{\textbf{Evolution of SVD of weight matrices for deep low-rank adaptation.} We visualize the SVD dynamics of an $L=3$ layer deep matrix factorization's end-to-end product employed for fine-tuning the 11th layer value matrix in BERT, with $d = 768$, $\epsilon_l = 1$ throughout Adam. \textit{Left}: Magnitude of the $i$-th singular value $\sigma_i(t)$ at iteration $t$. \textit{Middle}: Angle $\angle(\bm v_i(t), \bm v_i(0))$ between the $i$-th right singular vector at iteration $t$ and initialization. \textit{Right}: Angle $\angle(\bm u_i(t), \bm u_i(0))$ between the $i$-th left singular vector at iteration $t$ and initialization.}
    \label{fig:svd_lora_2d}
\end{figure}

\begin{figure*}[h]
\begin{center}
\includegraphics[width=\linewidth]{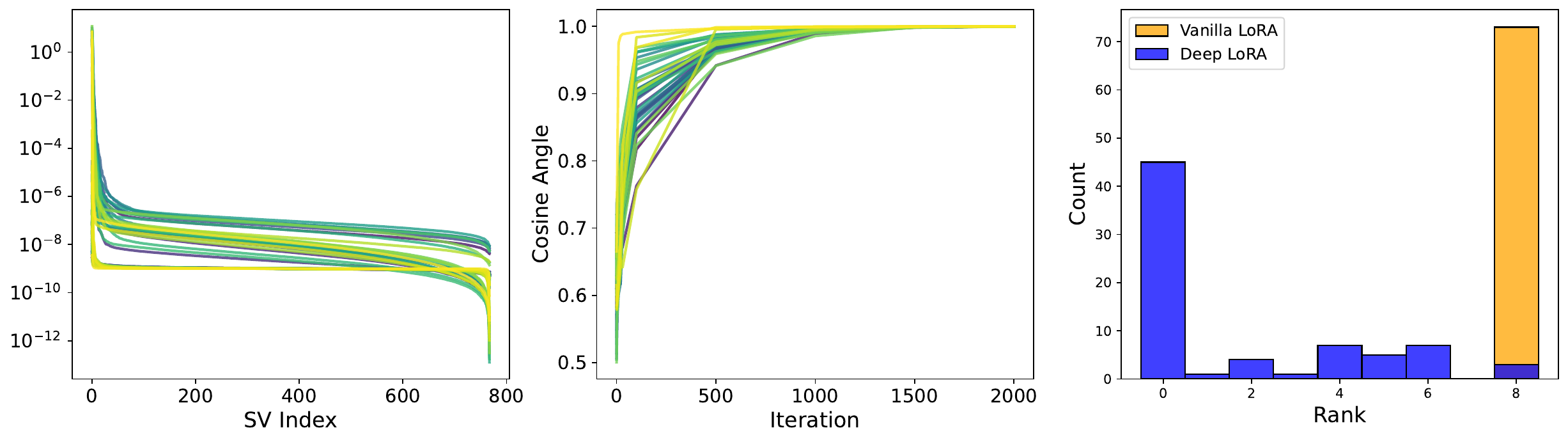}
\end{center}
\caption{\textbf{Invariant low-dimensional subspaces in deep overparameterized adaptation of language models.} Fine-tuning BERT with deep overparameterized adaptation on the STS-B dataset. \textit{Left}: \textbf{Singular value} spectra across all adapted layers at the end of fine-tuning. \textit{Middle}: \textbf{Alignment of subspaces} formed by the top 8 right singular vectors between the current adapted weights and the final adapted weights throughout training. \textit{Right}: \textbf{Deep LoRA finds lower ranked adapters} as opposed to LoRA, allowing it to have better test loss in the sample-starved regime.}
\label{fig:low_rank_subspace_2d}
\end{figure*}

\end{document}